\documentclass{new_tlp}
\pdfoutput=1

\usepackage[utf8]{inputenc}
\usepackage{amsmath,amsfonts,amssymb}
\usepackage{microtype}
\usepackage[center]{caption}

\usepackage{url}

\usepackage{stmaryrd}
\usepackage{listings}
\usepackage{multirow}
\usepackage{booktabs}

\newcommand{\naf}[1]{\ensuremath{{\neg{#1}}}}

\newcommand{\sig}[1]{\ensuremath{\mathcal{#1}}}

\newtheorem{theorem}{Theorem}[section]

\newcommand{\atoms}[1]{\ensuremath{atoms(#1)}}
\newcommand{\open}[0]{\emph{Open}}
\newcommand{\generate}[0]{\emph{Gen}}
\newcommand{\define}[0]{\emph{Def}}
\newcommand{\test}[0]{\emph{Test}}

\newcommand{\trf}[0]{\ensuremath{\Phi}}
\newcommand{\planif}[2]{\ensuremath{#1_#2}}
\newcommand{\planifnot}[2]{\ensuremath{#1_{\ifnot{#2}}}}
\newcommand{\ifnot}[1]{\ensuremath{\overline{#1}}}

\newcommand{\flag}[0]{\ensuremath{\alpha}}

\newcommand{\fixbf}[2]{\ensuremath{\mathit{fixbf}(#1,#2)}}
\newcommand{\fixcons}[2]{\ensuremath{\mathit{fixcons}(#1,#2)}}
\newcommand{\fixfact}[2]{\ensuremath{\mathit{fixfact}(#1,#2)}}

\newcommand{\dynx}[0]{\ensuremath{\mathbf{D}}}
\newcommand{\initx}[0]{\ensuremath{\mathbf{I}}}
\newcommand{\finalx}[0]{\ensuremath{\mathbf{G}}}
\newcommand{\dyn}[1]{\ensuremath{\dynx{}(#1)}}
\newcommand{\init}[1]{\ensuremath{\initx{}(#1)}}
\newcommand{\final}[1]{\ensuremath{\finalx{}(#1)}}
\newcommand{\prev}[1]{\ensuremath{#1'}}
\newcommand{\ttt}[1]{\ensuremath{#1^{\bullet}}}

\newcommand{\ttx}[1]{\ensuremath{#1^{\circ}}}
\newcommand{\pp}[0]{\ensuremath{\mathcal{PP}}}
\newcommand{\dd}[0]{\ensuremath{\mathcal{D}}}
\newcommand{\cpp}[0]{\ensuremath{\mathit{C}_{\dd}}}
\newcommand{\cp}[0]{\ensuremath{\mathit{C}_{\dd}}}

\newcommand{\copyinit}[1]{\ensuremath{copy}}

\newcommand{\assumeatoms}[0]{\ensuremath{\mathit{Assume}}}
\newcommand{\exampleemph}[1]{\emph{#1}}
\newcommand{\facts}[0]{\ensuremath{\mathit{Domain}}}
\newcommand{\base}[0]{\facts{}}
\newcommand{\guess}[0]{\ensuremath{\mathit{Guess}}}
\newcommand{\cone}[0]{\ensuremath{\mathit{C1}}}
\newcommand{\ctwo}[0]{\ensuremath{\mathit{C2}}}

\newcommand{\lpnormal}[0]{{\sc lp2normal}}
\newcommand{\lpacyc}[0]{{\sc lp2acyc}}
\newcommand{\lpsat}[0]{{\sc lp2sat}}
\newcommand{\qasp}[0]{{\sc qasp2qbf}}
\newcommand{\optt}[0]{{\sc opt}}
\newcommand{\fl}[0]{{\sc sat}}

\newcommand{\caqe}[1]{{\sc caqe}\textsuperscript{{\tiny \sc{#1}}}}
\newcommand{\caqecirc}[0]{{\sc caqe}\textsuperscript{$\circ$}}
\newcommand{\qesto}[1]{{\sc qesto}\textsuperscript{{\tiny \sc{#1}}}}
\newcommand{\qestocirc}[0]{{\sc qesto}\textsuperscript{$\circ$}}
\newcommand{\qute}[1]{{\sc qute}\textsuperscript{{\tiny \sc{#1}}}}
\newcommand{\qutecirc}[0]{{\sc qute}\textsuperscript{$\circ$}}
\newcommand{\depqbf}[1]{{\sc depqbf}\textsuperscript{{\tiny \sc{#1}}}}
\newcommand{\depqbfcirc}[0]{{\sc depqbf}\textsuperscript{$\circ$}}
\newcommand{\sclingo}[1]{{\sc clingo}\textsuperscript{{\tiny \sc{#1}}}}
\newcommand{\bloqqer}[0]{{\sc bloqqer}}
\newcommand{\hqspre}[0]{{\sc hqspre}}
\newcommand{\qratpre}[0]{{\sc qratpre}+}

\newcommand{\set}[1]{\ensuremath{\{#1\}}}

\def\cpasp{{\sc CPasp}}

\def\cpasp{{\sc CPasp}}

\def\cplan{$\cal C$-{\sc Plan}}

\def\cf2cff{{\tt cf2cs(ff)}}

\def\t0{{\tt t0}}
\def\T1{{\tt t1}}

\def\cplan{{\sc C-Plan}}

\newcommand{\dlvk}{$\mathtt{dlv}^\mathcal{K}$}

\newtheorem{lemma}{Lemma}

\begin{document}

\title{Planning with Incomplete Information in~Quantified~Answer Set Programming}

\author[Jorge Fandinno et al]{%
  Jorge Fandinno$^{2,3}$ \ François Laferrière$^3$ \ Javier Romero$^3$ \ Torsten Schaub$^3$ \ Tran Cao Son$^1$
  \\$^1${New Mexico State University, USA} \ $^2${Omaha State University, USA} \ $^3${University of Potsdam, Germany} }

\maketitle

\begin{abstract}
We present a general approach to planning with incomplete information in Answer Set Programming (ASP).
More precisely, we consider the problems of conformant and conditional planning with sensing actions and assumptions.
We represent planning problems using a simple formalism
where logic programs describe
the transition function between states, the initial states and the goal states.
For solving planning problems, we use Quantified Answer Set Programming (QASP),
an extension of ASP with existential and universal quantifiers over atoms
that is analogous to Quantified Boolean Formulas (QBFs).
We define the language of quantified logic programs
and use it to represent the solutions to different variants of conformant and conditional planning.
On the practical side, we present a translation-based QASP solver that
converts quantified logic programs into QBFs and then executes a QBF solver, and
we evaluate experimentally the approach on conformant and conditional planning benchmarks.
Under consideration for acceptance in TPLP.
\end{abstract}
%

\section{Introduction}\label{sec:introduction}

We propose a general and uniform framework for planning in Answer Set Programming (ASP;~\citeNP{lifschitz02a}).
Apart from classical planning,
referring to planning with deterministic actions and complete initial states,
our focus lies on conformant and conditional planning.
While the former extends the classical setting by incomplete initial states,
the latter adds sensing actions and conditional plans.
Moreover, we allow for making assumptions to counterbalance missing information.

To illustrate this, let us consider the following example.
There is a cleaning robot in a corridor going through adjacent rooms
that may be occupied by people.
The robot can go to the next room and
can sweep its current room to clean it, but
it should not sweep a room if it is occupied.
We assume that nothing changes if
the robot tries to go further than the last room, or
if it sweeps a room that is already clean.
The goal of the robot is to clean all rooms that are not occupied by people.
We present a solution for any number of rooms,
but in our examples we consider only two.

\emph{Classical planning.}
Consider an initial situation where
the robot is in the first room, only the first room is clean, and no room is occupied.
In this case, the classical planning problem is to
find a plan that, applied to \emph{the} initial situation, achieves the goal.
The plan,
where the robot goes to the second room and then sweeps it,
solves this problem.

\emph{Conformant planning.}
Consider now
that the robot initially does not know whether the rooms are clean or not.
There are four possible initial situations,
depending on the state of cleanliness of the two rooms.
In this case, the conformant planning problem is to
find a plan that, applied to \emph{all} possible initial situations,
achieves the goal.
The plan,
where the robot sweeps, goes to the second room and sweeps it,
solves that problem.

So far rooms were unoccupied.
Now consider that initially the robot knows that
exactly one of the rooms is occupied, but does not know which.
Combining the previous four options about cleanliness
with the two new options about occupancy,
there are eight possible initial situations.
It turns out that there is no conformant plan for this problem.
The robot would have to leave the occupied room as it is and sweep the other,
but there is no way of doing that without knowing which is the occupied room.

\exampleemph{Assumption-based planning.}
At this point,
the robot can try to make assumptions about the unknown initial situation,
and find a plan that at least works under these assumptions,
hoping that they will indeed hold in reality.
In this case, a conformant planning problem with assumptions is to
find a plan and a set of assumptions such that
the assumptions hold in some possible initial situation, and
the plan, applied to
\emph{all} possible initial situations satisfying the assumptions,
achieves the goal.
Assuming that room one is occupied,
the plan where the robot goes to room two and then sweeps it,
solves the problem.
Another solution is to assume that the second room is occupied and simply sweep the first room.

\emph{Conditional planning.}
The robot has a safer approach, if
it can
observe the occupancy of a room and
prepare a different subplan for each possible observation.
This is similar to conformant planning,
but now plans have actions to observe the world
and different subplans for different observations.
In our example, there is a conditional plan where
the robot first observes if the first room is occupied,
and if so, goes to the second room and sweeps it,
otherwise it simply sweeps the first room.
The robot could also make some assumptions about the initial situation,
but this is not needed in our example.

Unfortunately, the expressiveness of regular ASP is insufficient to capture this variety of planning problems.
While bounded classical and conformant planning are still expressible
since their corresponding decision problems are still
at the first and second level of the polynomial hierarchy,
bounded conditional planning is \textsc{Pspace}-complete~\cite{turner02a}.

To match this level of complexity, we introduce a quantified extension of ASP, called Quantified Answer Set Programming (QASP),
in analogy to Quantified Boolean Formulas (QBFs).
More precisely,
we start by capturing the various planning problems within a simple uniform framework
centered on the concept of transition functions,
mainly by retracing the work of~\citeANP{sotugemo05b}~\citeyear{sotugemo05b,tusoba07a,tusogemo11a}.
The core of this framework consists of a general yet simplified fragment of logic programs
that aims at representing transition systems,
similar to action languages~\cite{gellif98a} and (present-centered) temporal logic programs~\cite{cakascsc18a},
We then extend the basic setting of ASP with quantifiers and define quantified logic programs, in analogy to QBFs.
Although we apply QASP to planning problems, it is of general nature.
This is just the same with its implementation,
obtained via a translation of quantified logic programs to QBFs.
This allows us to represent the above spectrum of planning problems by quantified logic programs
and to compute their solutions with our QASP solver.
Interestingly, the core planning problems are described by means of the aforementioned simple language fragment,
while the actual type of planning problem is more or less expressed via quantification.
Finally,
we empirically evaluate our solver on conformant and conditional planning benchmarks.


\section{Background}\label{sec:background}

%
We consider normal logic programs over a set \sig{A} of atoms with choice rules and integrity constraints.
A rule $r$ has the form $H \leftarrow B$
where $B$ is a set of literals,
and $H$ is  either
an atom $p$, and we call $r$ a normal rule, or
$\{p\}$ for some atom $p$, making $r$ a choice rule, or
$\bot$, so that $r$ an integrity constraint.
We usually drop braces from rule bodies $B$, and
also use $l,B$ instead of $\{l\} \cup B$ for a literal $l$.
We also abuse notation and identify sets of atoms $X$
with sets of facts $\{ p \leftarrow {} \mid p \in X \}$.
A (normal) logic program is a set of (normal) rules.
%
%
As usual, rules with variables are viewed as shorthands for the set of their ground instances.
We explain further practical extensions of ASP, like conditional literals and cardinality constraints, in passing.
Semantically,
we identify a body $B$ with the conjunction of its literals,
the head of a choice rule $\{ p \}$ with the disjunction $p \vee \neg p$,
a rule $H \leftarrow B$ with the implication $B \rightarrow H$, and
a program with the conjunction of its rules.
A set of atoms $X\subseteq\sig{A}$ is a stable model of a logic program $P$ if
it is a subset-minimal model of the formula that results
from replacing in $P$ any literal by $\bot$ if it is not satisfied by $X$.
We let $\mathit{SM}(P)$ stand for the set of stable models of $P$.

The dependency graph of a logic program $P$ has nodes \sig{A},
an edge $p\stackrel{+}{\rightarrow}q$ if there is a rule
whose head is either $q$ or $\{q\}$ and whose body contains $p$, and
an edge $p\stackrel{-}{\rightarrow}q$ if there is a rule
with head $q$ or $\{q\}$, and $\naf{p}$ in its body.
A logic program is stratified if its dependency graph has no cycle involving a negative edge ($\stackrel{-}{\rightarrow}$).
Note that stratified normal programs have exactly one stable model,
unlike more general stratified programs.

ASP rests on a Generate-define-test (GDT) methodology~\cite{lifschitz02a}.
Accordingly, we say that a logic program $P$ is in GDT form if
it is stratified and
all choice rules in $P$ are of form $\{p\} \leftarrow {}$ such that
$p$ does not occur in the head of any other rule in $P$.
In fact,
GDT programs constitute a normal form because
every logic program can be translated into GDT form by using 
auxiliary atoms~\cite{niemela08b,famirosc20a}.

Quantified Boolean formulas (QBFs; \citeNP{gimana09a})
 extend propositional formulas
by existential ($\exists$) and universal ($\forall$) quantification over atoms.
We consider QBFs over \sig{A} of the form
\begin{equation}\label{qbf:formula}
Q_0 X_0 \ldots Q_n X_n \phi
\end{equation}
where $n \geq 0$,
$X_0$, \ldots, $X_n$ are pairwise disjoint subsets of \sig{A},
every $Q_i$ is either $\exists$ or $\forall$, and
$\phi$ is a propositional formula over \sig{A} in conjunctive normal form (CNF).
QBFs as in \eqref{qbf:formula} are in prenex conjunctive normal form.
More general QBFs can be transformed to this form in a 
satisfiability-preserving way~\cite{gimana09a}.
Atoms in $X_i$ are existentially (universally) quantified if
$Q_i$ is $\exists$ ($\forall$).
Sequences of quantifiers and sets $Q_0 X_0 \ldots Q_n X_n$ are called prefixes,
and abbreviated by $\mathcal{Q}$.
With it, a QBF as in (\ref{qbf:formula}) can be written as $\mathcal{Q}\phi$.
As usual, we represent CNF formulas as sets of clauses, and clauses as sets of literals.
For sets $X$ and $Y$ of atoms such that $X \subseteq Y$,
we define $\fixbf{X}{Y}$ as the set of clauses
$\{ \{p\} \mid p \in X\} \cup \{ \{ \neg p \} \mid p \in Y \setminus X \}$
that selects models containing the atoms in $X$ and no other atom from $Y$.
That is, if $\phi$ is a formula then the models of $\phi \cup \fixbf{X}{Y}$ are
$\{ M \mid M \textnormal{ is a model of }\phi \textnormal{ and } M \cap Y = X\}$.
%
%
Given that a CNF formula is satisfiable if it has some model,
the satisfiability of a QBF can be defined as follows:
\begin{itemize}
\item
$\exists X \phi$ is satisfiable if $\phi \cup \fixbf{Y}{X}$ is satisfiable for some $Y \subseteq X$.
\item
$\forall X \phi$ is satisfiable if $\phi \cup \fixbf{Y}{X}$ is satisfiable for all $Y \subseteq X$.
\item
$\exists X \mathcal{Q} \phi$ is satisfiable if $\mathcal{Q}\big( \phi \cup \fixbf{Y}{X}\big)$ is satisfiable for some $Y \subseteq X$.
\item
$\forall X \mathcal{Q} \phi$ is satisfiable if $\mathcal{Q}\big( \phi \cup \fixbf{Y}{X}\big)$ is satisfiable for all $Y \subseteq X$.
\end{itemize}
The formula $\phi$ in $\mathcal{Q}\phi$ generates a set of models,
while the prefix $\mathcal{Q}$ can be interpreted as a kind of query over them.
Consider $\phi_1 = \{\{a, b, \neg c\}, \{c \}\}$
and its models $\{\{a,c\}, \{a,b,c\}, \{b,c\}\}$.
Adding the prefix $\mathcal{Q}_1 = \exists\{a\}\forall\{b\}$ amounts to querying 
if there is some subset $Y_1$ of $\{a\}$ such that for all subsets $Y_2$ of $\{b\}$
there is some model of $\phi_1$ that contains the atoms in $Y_1 \cup Y_2$ and no other atoms from $\{a,b\}$.
The answer is yes, for $Y_1=\{a\}$, hence $\mathcal{Q}_1\phi_1$ is satisfiable.
One can check that letting $\mathcal{Q}_2$ be $\exists\{b,c\}\forall\{a\}$ it holds that
$\mathcal{Q}_2\phi_1$ is satisfiable,
while letting $\mathcal{Q}_3$ be $\exists\{a\}\forall\{b,c\}$ we have that $\mathcal{Q}_3\phi_1$ is not.
%

\section{Planning Problems}

In this section, we define different planning problems
with deterministic and non-concurrent actions
using a transition function approach
building on the work of~\citeN{tusoba07a}.

The domain of a planning problem is described in terms of
fluents, i.e.\ properties changing over time,
normal actions, and
sensing actions for observing fluent values.
We represent them by disjoint sets
$F$, $A_n$, and $A_s$ of atoms, respectively,
let $A$ be the set $A_n \cup A_s$ of actions,
and assume that $F$ and $A$ are non-empty.
For clarity, we denote sensing actions in $A_s$ by $a^f$ for some $f \in F$,
indicating that $a^f$ observes fluent $f$.
To simplify the presentation,
we assume that sets $F$, $A_n$, $A_s$ and $A$ are fixed.
A state $s$ is a set of fluents, $s \subseteq F$, that represents a snapshot of the domain.
To describe planning domains, we have to specify
what is the next state after the application of actions.
Technically, this is done by a transition function \trf, that is,
%
a function that takes as arguments a state and an action,
and returns either one state or the bottom symbol $\bot$.
Formally,
$\trf \colon \mathcal{P}(F) \times A \to
             \mathcal{P}(F) \cup \{ \bot \}$,
where $\mathcal{P}(F)$ denotes the power set of $F$.
The case where $\trf(s,a) = \bot$
represents that action $a$ is not executable in state $s$.
%
%

\emph{Example.}
Let $R$ be the set of rooms $\{1, \ldots, r\}$.
We represent our example domain with
the fluents $F = \{ at(x), clean(x), occupied(x) \mid x \in R \}$,
normal actions $A_n = \{ go, sweep \}$, and
sensing actions $A_s = \{sense(occupied(x)) \mid x \in R\}$.
For $r=2$, $s_1 = \{at(1), clean(1) \}$ is the state representing the initial situation
of our classical planning example.
%
%
The transition function $\trf_e$ can be defined as follows:
$\trf_e(s,go)$ is $(s \setminus \{at(x) \mid x \in R\}) \cup \{at(x+1) \mid at(x) \in s, x < r\}
\cup \{ at(r) \mid at(r) \in s \}$,
$\trf_e(s,sweep)$ is $s \cup \{clean(x) \mid at(x) \in s\}$ if, for all $x \in R$,
$at(x) \in s$ implies $occupied(x) \notin s$, and is $\bot$ otherwise; and,
for all $x \in R$, $\trf_e(s,sense(occupied(x)))$ is $s$ if $at(x) \in s$ and is $\bot$ otherwise.
%
%

Once we have fixed the domain,
we can define a planning problem as a tuple
$\langle \trf, I, G \rangle$
where $\trf$ is a transition function,
and $I,G \subseteq \mathcal{P}(F)$ are non-empty sets of initial and goal states.
A conformant planning problem is a planning problem with no sensing actions, viz.\ $A_s=\emptyset$,
and a classical planning problem is a conformant planning problem where $I$ is a singleton.
A planning problem with assumptions is then a tuple
$\langle \trf, I, G, As \rangle$
where $\langle \trf, I, G\rangle$ is a planning problem
and $As \subseteq F$ is a set of possible assumptions.

\emph{Example. }
The initial situation is $I_1 = \{s_1\}$ for our classical planning problem,
$I_2 = \{\{at(1)\}\cup X \mid X \subseteq \{clean(1),clean(2)\}\}$ for the first conformant planning problem and
$I_3 = \{ X \cup Y \mid X \in I_2, Y \in \{ \{ occupied(1) \}, \{ occupied(2) \} \} \}$ for the second one.
All problems share the goal states $G_e=
\{ s \subseteq F \mid  \textnormal{ for all } x \in R
  \textnormal{ either } occupied(x) \in s
  \textnormal{ or } clean(x) \in s\}$.
Let
$\pp_1$ be $\langle \trf_e, I_1, G_e \rangle$,
$\pp_2$ be $\langle \trf_e, I_2, G_e \rangle$, and
$\pp_3$ be $\langle \trf_e, I_3, G_e \rangle$.
If we disregard sensing actions (and adapt $\trf_e$ consequently)
these problems correspond
to our examples of classical and conformant planning, respectively.
The one of assumption-based planning is given by
$\pp_4 = \langle \trf_e, I_3, G_e, \{occupied(1),occupied(2)\} \rangle$, and
the one of conditional planning by $\pp_3$ with sensing actions.

Our next step is to define the solutions of a planning problem.
For this, we specify what is a plan
and extend transitions functions to apply to plans and sets of states.
A plan and its length are defined inductively as follows:
%
\begin{itemize}
\item
$[]$ is a plan, denoting the empty plan of length $0$.
\item
If $a\in A_n$ is a non-sensing action and $p$ is a plan,
then $[a;p]$ is a plan of length one plus the length of $p$.
\item
If $a^f \in A_s$ is a sensing action, and \planif{p}{f}, $\planifnot{p}{f}$ are plans,
then $[a^f; (p_f, p_{\overline{f}}) ]$ is a plan
of length one plus the maximum of the lengths of $\planif{p}{f}$ and $\planifnot{p}{f}$.
\end{itemize}
We simplify notation and write $[a;[]]$ as $[a]$,
and $[a;[\sigma]]$ as $[a;\sigma]$ for any action $a$ and plan $[\sigma]$.
For example,
$p_1 = [go; sweep]$,
$p_2 = [sweep; go; sweep]$, and
$p_3 = [sense(occupied(1));$ $([go;clean],[clean])]$ are plans
of length $2$, $3$, and $3$, respectively.

We extend the transition function $\trf$ to a set of states $S$
as follows:
$\trf(S,a)$ is $\bot$ if there is some $s \in S$ such that $\trf(a,s)=\bot$,
and is $\bigcup_{s \in S} \trf(s,a)$ otherwise.
%
%
In our example,
$\trf_e(I_1,go) = \{\{at(2),clean(1)\}\}$,
$\trf_e(I_2,sweep) = \{\{at(1),clean(1)\},\{at(1),clean(1),$ $clean(2)\}\}$,
$\trf_e(I_3,sweep) = \bot$, and
$\trf_e(\trf_e(I_3,go),sweep) = \bot$.
%
%
With this, we can extend the transition function $\trf$ to plans as follows.
Let
$p$ be a plan, and $S$ a set of states, then:
\begin{itemize}
\item
If $p=[]$ then $\trf(S,p)=S$.
\item
If $p=[a;q]$, where $a$ is a non-sensing action and $q$ is a plan, then
\[
\trf(S,p) =
\begin{cases}
\bot & \textnormal{if } \trf(S,a) = \bot \\
\trf(\trf(S,a),q) & \textnormal{otherwise}
\end{cases}
\]
\item
If
$p = [a^f; (\planif{q}{f}, \planifnot{q}{f})]$,
where $a^f$ is a sensing action and
$\planif{q}{f}$, $\planifnot{q}{f}$ are plans,
then
\[
\trf(S,p) =
\begin{cases}
    \bot & \textnormal{if either } \trf(S,a^f), \trf(S^f,\planif{q}{f}) \textnormal{ or } \trf(S^{\ifnot{f}},\planifnot{q}{f}) \textnormal{ is } \bot \\
\trf(S^f          ,\planif   {q}{f})  \cup
\trf(S^{\ifnot{f}},\planifnot{q}{f})
& \textnormal{otherwise}
\end{cases}
\]
where $S^f=\{s \mid f    \in s, s \in \trf(S,a^f) \}$ and
      $S^{\ifnot{f}}=\{s \mid f \notin s, s \in \trf(S,a^f) \}$.
\end{itemize}
In our example,
$\trf(I_1,p_1) = \{\{at(2),clean(1),clean(2)\}\}$,
$\trf(I_2,p_2) = \{\{at(2), clean(1),$ $clean(2)\}\}$,
$\trf(I_3,q) = \bot$ for any plan $q$ without sensing actions that involves some $sweep$ action,
$\trf(\{s \in I_3 \mid occupied(1) \in s\},p_1) =
\{ \{at(2),occupied(1),clean(2)\} \cup X \mid X \subseteq \{clean(1)\}\}$, and
$\trf(I_3,p_3) =
\{ \{at(2),occupied(1),clean(2)\} \cup X \mid X \subseteq \{clean(1)\}\} \cup
\{ \{at(1),occupied(2),clean(1)\} \cup X \mid X \subseteq \{clean(2)\}\}$.

We can now define the solutions of planning problems:
a plan $p$ is a solution to planning problem
$\langle \trf, I, G \rangle$
if $\trf(I,p)\neq\bot$ and $\trf(I,p) \subseteq G$.
In our example,
plan $p_1$ solves $\pp_1$,
$p_2$ solves $\pp_2$, and
$p_3$ solves $\pp_3$.
There is no plan without sensing actions solving $\pp_3$.
For assumption-based planning, a tuple $\langle p, T, F \rangle$,
where $p$ is a plan and $T,F \subseteq As$,
is a solution to a planning problem with assumptions
$\langle \trf, I, G, As \rangle$
if
(1) $J = \{s \mid s \in I, T \subseteq s, s \cap F = \emptyset\}$ is not empty, and
(2) $p$ solves the planning problem $\langle \trf, J, G \rangle$.
Condition (1) guarantees that the true assumptions $T$ and the false assumptions $F$
are consistent with some initial state, and
condition (2) checks that $p$ achieves the goal starting from all initial states consistent with the assumptions.
%
%
For example,
the planning problem with assumptions $\pp_4$ is solved by
$\langle p_1, \{occupied(1)\},\{\}\rangle$
and by
$\langle [sweep], \{occupied(2)\},\{\}\rangle$.
%

\section{Representing Planning Problems in ASP}

In this section we present an approach for representing planning problems
using logic programs.
Let $F$, $A_n$, $A_s$ and $A$ be sets of atoms as before, and
let $F'=\{ f' \mid f \in F\}$ be a set of atoms that we assume to be disjoint from the others.
We use the atoms in $F'$ to represent the value of the fluents in the previous situation.
We represent planning problems by planning descriptions, 
that consist of dynamic rules to represent transition functions, 
initial rules to represent initial states, and
goal rules to represent goal states.
Formally, a dynamic rule is a rule whose head atoms belong to $F$ and whose body atoms belong to $A \cup F \cup F'$,
an initial rule is a rule whose atoms belong to $F$,
and a goal rule is an integrity constraint whose atoms belong to $F$.
%
%
Then a planning description $\dd$ is a tuple $\langle \mathit{DR}, \mathit{IR}, \mathit{GR} \rangle$ of
dynamic rules $\mathit{DR}$, initial rules $\mathit{IR}$, and goal rules $\mathit{GR}$.
By $\dyn{\dd}$, $\init{\dd}$ and $\final{\dd}$ we refer to the elements $\mathit{DR}$, $\mathit{IR}$ and $\mathit{GR}$, respectively.
%

\emph{Example.}
We represent the transition function $\trf_e$ by the following dynamic rules $\mathit{DR}_e$:
\begin{equation*}
\begin{array}{r@{\ }c@{\ }lr@{\ }c@{\ }l}
at(R) & \leftarrow & go, \prev{at(R-1)}, R<r
 & 
clean(R) & \leftarrow & sweep, \prev{at(R)}\\
at(r) & \leftarrow & go, \prev{at(r)}
 &
clean(R) & \leftarrow & \prev{clean(R)}\\
at(R) & \leftarrow & \naf{go}, \prev{at(R)}
 &
occupied(R) & \leftarrow & \prev{occupied(R)}
\\
\bot & \leftarrow & sweep, \prev{at(R)}, \prev{occupied(R)}
 &
\bot & \leftarrow & sense(occupied(R)),  \naf{\prev{at(R)}}
\end{array}\end{equation*}%
%
On the left column, 
the normal rules describe the position of the robot 
depending on its previous position and the action $go$, while
the integrity constraint below forbids the robot 
to sweep if it is in a room that is occupied.
On the right column, the first two rules state when a room is clean, 
the third one expresses that rooms remain occupied if they were before, and
the last integrity constraint forbids the robot to observe a room if it is not at it.
%
%
For the initial states, $I_1$ is represented by the initial rules
$\mathit{IR}_1 = \{ at(1) \leftarrow ; clean(1) \leftarrow \}$, 
$I_2$ is represented by $\mathit{IR}_2$:
\begin{equation*}\begin{array}{r@{\ }c@{\ }lcr@{\ }c@{\ }l}
        at(1) & \leftarrow & \phantom{R=1..r} & &
\{ clean(R) \} & \leftarrow & R=1..r
\end{array}\end{equation*}
and $I_3$ by $\mathit{IR}_3$, that contains the rules in $\mathit{IR}_2$ and also these ones:
\begin{equation*}\begin{array}{r@{\ }c@{\ }lccr@{\ }c@{\ }l}
\{ occupied(R) \} & \leftarrow & R=1..r  & & &
\bot & \leftarrow & \{ occupied(R) \} \neq 1
\end{array}\end{equation*}
The choice rules generate the different possible initial states, and
the integrity constraint inforces that exactly one room is occupied.
The goal states $G$ are represented by $\mathit{GR}_e = \{
\bot \leftarrow \naf{occupied(R)}, \naf{clean(R)}, R=1..n\}$
that forbids states where some room is not occupied and not clean.

Next we specify formally the relation between planning descriptions and planning problems.
We extract a transition function $\trf(s,a)$
from a set of dynamic rules $\mathit{DR}$ by looking at the stable models 
of the program $s' \cup \{a \leftarrow\} \cup \mathit{DR}$ for every state $s$ and action $a$, 
where $s'$ is $\{f' \mid f \in s\}$.
The state $s$ is represented as $s'$ to stand for the previous situation, 
and the action and the dynamic rules generate the next states.
Given that we consider only deterministic transition functions, 
we restrict ourselves to deterministic problem descriptions, 
where we say that a set of dynamic rules $\mathit{DR}$ is deterministic if 
for every state $s \subseteq F$ and action $a \in A$, 
the program $s' \cup \{a\leftarrow\} \cup \mathit{DR}$ has zero or one stable model, 
and a planning description $\dd$ is deterministic if \dyn{\dd} is deterministic.
A deterministic set of dynamic rules $\mathit{DR}$ defines the transition function 
\[
\trf_{\mathit{DR}}(s,a) = 
\begin{cases}
M \cap F & \textnormal{ if } s' \cup \{a\} \cup \mathit{DR} \textnormal{ has a single stable model }M\\
\bot & \textnormal{otherwise}
\end{cases}
\]
defined for every state $s \subseteq F$ and action $a \in A$.
Note that given that $\mathit{DR}$ is deterministic,
the second condition only holds when the program $s' \cup \{a\} \cup \mathit{DR}$ has no stable models.
For the case where no action occurs, 
we require dynamic rules to make the previous state persist.
This condition is not strictly needed, 
but it makes the formulation of the solutions to planning problems in Section \ref{section:solving} easier.
Formally, we say that a set of dynamic rules $\mathit{DR}$ is inertial if 
for every state $s \subseteq F$ it holds that $SM(s' \cup \mathit{DR}) = \{s' \cup s\}$, and
we say that a planning description $\dd$ is inertial if \dyn{\dd} is inertial.
From now on,
we restrict ourselves to deterministic and inertial planning descriptions.

Coming back to the initial and goal rules of a planning description $\dd$,
the first ones represent the initial states $SM(\init{\dd})$, 
while the second ones represent the goal states $SM(\{\{f\}\leftarrow \mid f \in F\}\cup \final{\dd})$.
In the latter case,
the choice rules generate all possible states while 
the integrity constraints in \final{\dd} eliminate those that are not goal states.
Given that we consider only non-empty subsets of initial and goal states, 
we require the programs
\init{\dd} and $\{\{f\}\leftarrow \mid f \in F\}\cup \final{\dd}$
to have at least one stable model.
Finally, putting all together, we say that
a deterministic and inertial planning description $\dd$ represents the planning problem
$\langle \trf_{\dyn{\dd}}, SM(\init{\dd}), SM( \{\{f\}\leftarrow \mid f \in F\}\cup \final{\dd}) \rangle$.
Moreover, the planning description $\dd$ together with a set of atoms $As$
represent the planning problem with assumptions
$\langle \trf_{\dyn{\dd}}, SM(\init{\dd}), SM( \{\{f\}\leftarrow \mid f \in F\}\cup \final{\dd}), As \rangle$.

\emph{Example}.
One can check that the dynamic rules $\mathit{DR}_e$
are deterministic, inertial, and define the transition function $\trf_e$ of the example.
Moreover, 
$\dd_1 = \langle \mathit{DR}_e, \mathit{IR}_1, \mathit{GR}_e \rangle$ represents $\pp_1$,  
$\dd_2 = \langle \mathit{DR}_e, \mathit{IR}_2, \mathit{GR}_e \rangle$ represents $\pp_2$,
$\dd_3 = \langle \mathit{DR}_e, \mathit{IR}_3, \mathit{GR}_e \rangle$ represents $\pp_3$, and
$\dd_3$ with the set of atoms $\{occupied(1), occupied(2)\}$ represents $\pp_4$.

\section{Quantified Answer Set Programming}

Quantified Answer Set Programming (QASP) is an extension of ASP to quantified logic programs (QLPs),
analogous to the extension of propositional formulas by QBFs.
A quantified logic program over \sig{A} has the form
\begin{equation}\label{qasp:program}
Q_0 X_0 \ldots Q_n X_n P
\end{equation}
where $n \geq 0$,
$X_0$, \ldots, $X_n$ are pairwise disjoint subsets of \sig{A},
every $Q_i$ is either $\exists$ or $\forall$, and
$P$ is a logic program over \sig{A}.
Prefixes are the same as in QBFs, and we may refer to a QLP as in (\ref{qasp:program}) by $\mathcal{Q}P$.
For sets $X$ and $Y$ of atoms such that $X \subseteq Y$,
we define $\fixcons{X}{Y}$ as the set of rules
\(
\{\bot \leftarrow \neg{x} \mid x \in X \} \cup \{\bot \leftarrow x \mid x \in Y \setminus X \}
\)
that selects stable models containing the atoms in $X$ and no other atom from $Y$.
That is, if $P$ is a logic program then the stable models of $P \cup \fixcons{X}{Y}$ are
$\{ M \mid M \textnormal{ is a stable model of } P \textnormal{ and } M \cap Y = X\}$.
Given that a logic program is satisfiable if it has a stable model,
the satisfiability of a QLP is defined as follows:
\begin{itemize}
\item $\exists X P$ is satisfiable if program $P \cup \fixcons{Y}{X}$ is satisfiable for some $Y \subseteq X$.
\item $\forall X P$ is satisfiable if program $P \cup \fixcons{Y}{X}$ is satisfiable for all $Y \subseteq X$.
\item $\exists X \mathcal{Q} P$ is satisfiable if program $\mathcal{Q}\big( P \cup \fixcons{Y}{X}\big)$ is satisfiable for some $Y \subseteq X$.
\item $\forall X \mathcal{Q} P$ is satisfiable if program $\mathcal{Q}\big( P \cup \fixcons{Y}{X}\big)$ is satisfiable for all $Y \subseteq X$.
\end{itemize}
As with QBFs, program $P$ in $\mathcal{Q}P$ generates a set of stable models,
while its prefix $\mathcal{Q}$ can be seen as a kind of query on it.
Consider
\(
P_1=\{\{a\} \leftarrow {}; \{b\} \leftarrow {}; c \leftarrow a; c \leftarrow b; \bot \leftarrow \naf{c}\}
\)
and its stable models $\{a,c\}$, $\{a,b,c\}$, $\{b,c\}$.
The prefixes of $\mathcal{Q}_1 P_1$, $\mathcal{Q}_2 P_1$ and $\mathcal{Q}_3 P_1$
pose the same queries over the stable models of $P$ than those posed
in $\mathcal{Q}_1 \phi_1$, $\mathcal{Q}_2 \phi_1$ and $\mathcal{Q}_3 \phi_1$
over the models of $\phi_1$.
Given that the stable models of $P_1$ and the models of $\phi_1$ coincide,
the satisfiability of the $\mathcal{Q}_iP_1$'s is the same as that of
the corresponding $\mathcal{Q}_i\phi_1$'s.
This result is generalized by the following theorem that relates QLPs and QBFs.
%
\begin{theorem}\label{qasp:theorem:qbf}
Let $P$ be a logic program over \sig{A} and
$\phi$ be a CNF formula over $\sig{A}\cup\sig{B}$
such that $SM(P) = \{ M \cap \sig{A} \mid M \textnormal{ is a model of } \phi\}$.
For every prefix $\mathcal{Q}$ whose sets belong to \sig{A},
the QLP $\mathcal{Q}P$ is satisfiable if and only if the QBF $\mathcal{Q}\phi$ is satisfiable.
\end{theorem}
The proof is by induction on the number of quantifiers in $\mathcal{Q}$.
The condition $SM(P) = \{ M \cap \sig{A} \mid M \textnormal{ is a model of } \phi\}$
of Theorem \ref{qasp:theorem:qbf} is satisfied by
existing polynomial-time translations
from logic programs $P$ over \sig{A} to CNF formulas $\phi$ over 
$\sig{A}\cup\sig{B}$, and
from CNF formulas $\phi$ over $\sig{A}$ to logic programs $P$ over \sig{A}~\cite{janhunen04a}.
%
Using these translations, Theorem \ref{qasp:theorem:qbf}, and
the fact that deciding whether a QBF is satisfiable is PSPACE-complete,
we can prove the following complexity result about QASP.
\begin{theorem}\label{qasp:complexity}
The problem of deciding whether a given QLP $\mathcal{Q}P$ is satisfiable is PSPACE-complete.
\end{theorem}

The implementation of our system \qasp\footnote{\url{https://github.com/potassco/qasp2qbf}} relies on the previous results.
The input is a QLP $\mathcal{Q}P$ that is specified by putting together the rules of $P$
with facts over the predicates $\_exists/2$ and $\_forall/2$ describing the prefix $\mathcal{Q}$, 
where $\_exists(i,a)$ ($\_forall(i,a)$, respectively) 
asserts that the atom $a$ is existentially (universally, respectively) quantified at position $i$ of $\mathcal{Q}$.
The system first translates $P$ into a CNF formula $\phi$ that satisfies the condition of Theorem \ref{qasp:theorem:qbf}
using the tools \lpnormal, \lpacyc, and \lpsat,\footnote{%
\url{http://research.ics.aalto.fi/software/asp}} and
then uses a QBF solver to decide the satisfiability of $\mathcal{Q}\phi$.
If $\mathcal{Q}\phi$ is satisfiable and the outermost quantifier is existential,
then the system returns an assignment to the atoms occurring in the scope of that quantifier.

\section{Solving Planning Problems in QASP}
\label{section:solving}

In this section, 
we describe how to solve planning problems represented by a planning description 
using QASP.

\emph{Classical planning.}
We start with a planning description $\mathcal{\dd}$
that represents a classical planning problem $\pp = \langle \trf, \{s_0\}, G \rangle$.
Our task, given a positive integer $n$, 
is to find a plan $[a_1; \ldots; a_n]$ of length $n$ such that $\trf(\{s_0\},p) \subseteq G$.
%
This can be solved as usual in answer set planning~\cite{lifschitz02a}, 
using choice rules to generate possible plans, 
initial rules to define the initial state of the problem, 
dynamic rules replicated $n$ times to define the next $n$ steps, and
goal rules to check the goal conditions at the last step.
To do that in this context, 
we first let \facts\ be the union of the following sets of facts
asserting the time steps of the problem, the actions, 
and the fluents that are sensed by each sensing action:
$\{ t(T) \mid t \in \{1, \ldots, n\} \}$,
$\{ action(a) \mid a \in A \}$, and
$\{ senses(a^f,f) \mid a^f \in A_s\}$.
The last set is only needed for conditional planning, but
we already add it here for simplicity.
Given these facts, the following choice rule generates the possible plans of the problem:
\begin{equation}\label{solving:generate}
\{ occ(A,T) : action(A) \} = 1 \leftarrow t(T)
\end{equation}
Additionally, by \ttx{\dd_\initx} (\ttx{\dd_\finalx}, respectively) we denote the 
set of rules that results from replacing in \init{\dd} (\final{\dd}, respectively)
the atoms $f \in F$ by $h(f,0)$ (by $h(f,n)$, respectively);
by \ttx{\dd_\dynx} we denote the 
set of rules that results from replacing in \dyn{\dd} 
the atoms $f \in F$ by $h(f,T)$,
the atoms $f' \in F'$ by $h(f,T-1)$,
the atoms $a \in A$ by $occ(a,T)$ and 
adding the atom $t(T)$ to the body of every rule; and
by \ttx{\dd} we denote the program $\ttx{\dd_\initx} \cup  \ttx{\dd_\dynx} \cup \ttx{\dd_\finalx}$.
Putting all together,
the program $\facts \cup (\ref{solving:generate}) \cup \ttx{\dd}$
represents the solutions to the planning problem $\pp$.
The choice rule (\ref{solving:generate}) guesses plans $[ a_1; \ldots; a_n ]$ using 
atoms of the form $occ(a_1,1), \ldots, occ(a_n,n)$,
the rules of \ttx{\dd_\initx} define the initial state $s_0$ using atoms of the form $h(\cdot,0)$,
the rules of \ttx{\dd_\dynx} define 
the next states $s_i = \trf(s_{i-1},a_i)$ for $i\in\{1,\ldots,n\}$ using atoms of the form $h(\cdot,i)$ while
at the same time check that the actions $a_i$ are executable in $s_{i-1}$, and
the rules of \ttx{\dd_\finalx} check that the last state $s_n$ 
belongs to $G$.
%
Of course, all this works only because $\dd$ represents $\pp$ and therefore
$\trf$, $\{s_0\}$ and $G$ are defined by $\dyn{\dd}$, \init{\dd} and \final{\dd}, respectively.
%
Finally, letting $Occ$ be the set of atoms
$\{ occ(a,t) \mid a \in A, t \in \{1, \ldots, n\}\}$,
we can represent the solutions to \pp\ by the quantified logic program
\begin{equation}\label{solving:classical}
\exists Occ \big( \facts \cup (\ref{solving:generate}) \cup \ttx{\dd}\big)
\end{equation}
where the atoms selected by the existential quantifier correspond to solutions to $\pp$.
Going back to our example, where $\dd_1$ represents the problem $\pp_1$, 
we have that for $n=2$
the program (\ref{solving:classical}) (adapted to $\dd_1$)
is satisfiable selecting the atoms $\{occ(go,1), occ(sweep,2)\}$ 
that represent the solution $p_1$.

\emph{Conformant planning.}
When $\dd$ represents a conformant planning problem $\pp =\langle \trf, I, G \rangle$
our task is to find a plan $p$ such that $\trf(I,p) \subseteq G$,
or, alternatively, $p$ must be such that for all $s \in I$ it holds that $\trf(\{s\},p) \subseteq G$.
This formulation of the problem suggests to use a prefix $\exists \forall$
where the existential quantifier guesses a plan $p$ and the universal quantifier 
considers all initial states.
Let us make this more concrete.
From now on we assume that $\init{\dd}$ is in \emph{GDT} form.\footnote{%
Note that this implies that \ttx{\dd_{\initx}} is also in \emph{GDT} form.}
If that is not the case then we can translate the program into \emph{GDT} form
using the method mentioned in the Background section, and we continue from there.
Let \generate{}, \define{} and \test{}
be the choice rules, normal rules and integrity constraints of \ttx{\dd_{\initx}}, respectively, 
and let \open{} be the set $\{h(f,0) \mid \{h(f,0)\}\leftarrow \in \generate\}$
of possible guesses of the choice rules of \ttx{\dd_{\initx}}.
Observe that for every possible set $X \subseteq \open{}$
the program $X \cup \define{}$ has a unique stable model $M$, 
and if $M$ is also a model of \test{} then
$M$ is a stable model of \ttx{\dd_{\initx}}.
Moreover, note that all stable models of \ttx{\dd_{\initx}} can be constructed in this manner.
Given this, we say that the sets $X \subseteq \open{}$
that lead to a stable model $M$ of \ttx{\dd_{\initx}} are relevant, 
because they can be used as representatives of the initial states, 
and the other sets in \open{} are irrelevant.
Back to our quantified logic program, 
we are going to use the prefix $\exists Occ \forall \open{}$.
This works well with the logic program 
$\facts \cup (\ref{solving:generate}) \cup \ttx{\dd}$
whenever all choices $X \subseteq \open{}$ are relevant.
But it fails as soon as there are irrelevant sets
because, when we select them as subsets of \open{} in the universal quantifier,
the resulting program becomes unsatisfiable.
To fix this, we can modify our logic program 
so that for the irrelevant sets the resulting program becomes always satisfiable.
We do that in two steps.
First, we modify \ttx{\dd_{\initx}} so that 
the irrelevant sets lead to a unique stable model that contains the special atom $\flag(0)$.
This is done by the program \ttt{\dd_\initx} 
that results from replacing in \ttx{\dd_\initx}
the symbol $\bot$ in the head of the integrity constraints by $\flag(0)$.
Additionally, we consider the rule 
\begin{equation}\label{solving:flags}
\flag(T) \leftarrow t(T), \flag(T{-}1)
\end{equation}
that derives $\alpha(1)$, \ldots, $\alpha(n)$ for the irrelevant sets.
Second, we modify \ttx{\dd_{\dynx}} and \ttx{\dd_{\finalx}} so that 
whenever those special atoms are derived, these programs are inmediately satisfied.
This is done by the programs \ttt{\dd_\dynx} and \ttt{\dd_\finalx}
that result from adding 
the literal $\naf{\flag(T)}$
to the bodies of the rules in \ttx{\dd_\dynx} and \ttx{\dd_\finalx}, respectively.
Whenever the special atoms $\flag(0)$, \ldots, $\flag(n)$ are derived, 
they deactivate the rules in \ttt{\dd_\dynx} and \ttt{\dd_\finalx}
and make the program satisfiable.
We denote by \ttt{\dd} the program $\ttt{\dd_\initx} \cup  \ttt{\dd_\dynx} \cup \ttt{\dd_\finalx}$.
Then the following theorem establishes the correctness and completeness of the approach.
\begin{theorem}\label{theorem:confp} 
Let $\dd$ be a planning description that represents a conformant planning problem $\pp$, and $n$ be a positive integer.
If $\init{P}$ is in $GDT$ form, then
there is a plan of length $n$ that solves $\pp$ 
if and only if the following quantified logic program is satisfiable:
\begin{equation}\label{solving:conformant}
\exists Occ \forall \open{} \big( \facts \cup \ttt{\dd} \cup (\ref{solving:generate})\cup(\ref{solving:flags}) \big)
\end{equation}
\end{theorem}
In our example, where $\dd_2$ represents the problem $\pp_2$, 
for $n=3$
the program (\ref{solving:conformant}) (adapted to $\dd_2$)
is satisfiable selecting the atoms $\{occ(sweep,1), occ(go,2), occ(sweep,3)\}$ 
that represent the solution $p_2$,
while for $\dd_3$, that represents $\pp_3$, the
corresponding program is unsatisfiable for any integer $n$.

\emph{Assumption-based planning.}
Let 
$\dd$, together with a set of atoms $As \subseteq F$, 
represent a conformant planning problem with assumptions $\pp=\langle \trf, I, G , As\rangle$.
To solve this problem we have to find a plan $p=[a_1; \ldots; a_n]$ and
a set of assumptions $T,F \subseteq As$ such that
(1) the set $J = \{s \mid s \in I, T \subseteq s, s \cap F = \emptyset\}$ is not empty, and 
(2) $p$ solves the conformant planning problem $\langle \trf, J, G \rangle$.
The formulation of the problem suggests that first we can 
guess the set of assumptions $T$ and $F$, and
then check (1) and (2) separately.
The guess can be represented by the set of rules \guess\
that consists of the facts
$\{assumable(f) \mid f \in As\}$
and the choice rule
\[
\{ assume(F,true); assume(F,false) \} \leq 1 \leftarrow assumable(F)
\]
that generates all possible sets of assumptions using the predicate $assume/2$.
Moreover, we add the set of atoms $\assumeatoms = \{ assume(f,v) \mid f \in As, v \in \{true,false\}\}$ 
to $Occ$ at the outermost existential quantifier of our program.
Condition (1) can be checked by the set of rules \cone\ that can be divided in two parts.
The first part is a copy of the initial rules,
that consists of the rules that result from replacing in \init{P} every atom $f \in F$ by $init(f)$, 
and in this way generates all possible initial states in $I$ using the predicate $init/1$. 
The second part contains the integrity constraints 
\begin{equation*}\begin{array}{rclcccrcl}
\bot & \leftarrow & \naf{init(F)}, assume(F,true) & & &
\bot & \leftarrow &     {init(F)}, assume(F,false)
\end{array}\end{equation*}
that guarantee that the guessed assumptions represented by $assume/2$ 
hold in some initial state represented by $init/1$.
Condition (2) can be represented extending the program for conformant planning
with the following additional rules \ctwo, stating that 
the initial states that do not agree with the guessed assumptions are irrelevant:
\begin{equation*}
\begin{array}{rclcccrcl}
\flag(0) & \leftarrow & \naf{h(F,0)}, assume(F,true) & & &
\flag(0) & \leftarrow & \phantom{\naf{}}h(F,0), assume(F,false)
\end{array}\end{equation*}
With these rules, the plans only have to succeed 
starting from the initial states that agree with the assumptions, 
and condition (2) is satisfied.
\begin{theorem}\label{theorem:assump} 
Let $\dd$, $\pp$, and $As$ be as specified before, and $n$ be a positive integer.
If $\init{\dd}$ is in \emph{GDT} form, then
there is a plan with assumptions of length $n$ that solves $\pp$
if and only if the following quantified logic program is satisfiable:
\begin{equation}\label{solving:assump}
\exists ( Occ \cup \assumeatoms) \forall \open{} \big( \facts \cup \ttt{\dd} \cup (\ref{solving:generate})\cup(\ref{solving:flags}) \cup \guess \cup \cone \cup \ctwo \big)
\end{equation}
\end{theorem}
%
In our example, 
where $\dd_3$ together with the set of atoms $\{occupied(1), occupied(2)\}$ represents $\pp_4$, 
we have that for $n=2$
the program (\ref{solving:assump}) (adapted to $\dd_3$)
is satisfiable selecting the atoms $\{occ(go,1), occ(sweep,2), assume(occupied(1),true)\}$ 
that represent the solution $\langle p_1,\{occupied(1)\}, \{\}\rangle$,
and for $n=1$ selecting the atoms $\{occ(sweep,1), $ $assume(occupied(2),true)\}$ 
that represent the solution $\langle [sweep],\{occupied(2)\}, \{\}\rangle$.

\emph{Conditional planning.}
Consider the case where
$\dd$ represents a planning problem with sensing actions $PP=\langle \trf, I, G \rangle$.
Again, we have to find a plan $p$ such that $\trf(I,p) \subseteq G$, 
but this time $p$ has the form of a tree where sensing actions 
can be followed by different actions depending on sensing results.
%
%
This suggests a formulation where 
we represent the sensing result at time point $T$ by the truth value of an atom $obs(true,T)$, 
we guess those possible observations with a choice rule:
\begin{equation}\label{solving:obs}
\{ obs(true,T) \} \leftarrow t(T), T<n
\end{equation}
and, letting $Obs_t$ be $\{ obs(true,t) \}$ and $Occ_t$ be $\{ occ(a,t) \mid a \in A \}$ for $t \in \{1,\dotsc,n\}$,
we consider a QLP of the form
\[\exists Occ_1 \forall Obs_1 \ldots \exists Occ_{n-1} \forall Obs_{n-1} \exists Occ_n \forall \open{}
\big( \facts \cup \ttt{\dd} \cup (\ref{solving:generate})\cup(\ref{solving:flags})\cup(\ref{solving:obs}) \big).\]
%
%
This formulation nicely allows for a different plan $[a_1; \ldots; a_n]$ for every sequence $[O_1; \ldots; $ $O_{n-1}]$ 
of observations $O_i \subseteq Obs_i$. 
But at the same time it requires that each such plan achieves the goal for all possible initial situations, 
and this requirement is too strong.
Actually, we only want the plans to work for those cases
where, at every time step $T$, the result of a sensing action $a^f$, represented by $obs(true,T)$, 
coincides with the value of the sensed fluent represented by $h(f,T)$.
We can achieve this by signaling the other cases as irrelevant with the following rule:
\begin{equation}\label{solving:sensing}
\flag (T) \leftarrow t(T), occ(A,T{-}1), senses(A,F), \{ h(F,T{-}1), obs(true,T{-}1) \} = 1
\end{equation}
where the cardinality constraint  $\{ h(F,T{-}1), obs(true,T{-}1) \} = 1$
holds if the truth value of $obs(true,T{-}1)$ and $h(F,T{-}1)$ is not the same.
Another issue with the previous QLP is that
it allows normal actions at every time point $T$
to be followed by different actions at $T{+}1$ for each value of $obs(true,T)$.
This is not a problem for the correctness of the approach, 
but it is not a natural representation.
To fix this, we can consider that, 
whenever a normal action occurs at time point $T$, 
the case where $obs(true,T)$ holds is irrelevant:
\begin{equation}\label{solving:nonsensing}
\flag (T)  \leftarrow t(T), occ(A,T), \{ senses(A,F) \} = 0, obs(true,T)
\end{equation}
Appart from this, note that in conditional planning the different subplans may have different lengths.
For this reason, in the choice rule (\ref{solving:generate}) we have to replace the symbol `$=$' by `$\leq$'.
We denote the new rule by (\ref{solving:generate})${}^{\leq}$. 
\begin{theorem}\label{theorem:condp} 
Let $\dd$ and $\pp$ be as specified before, and $n$ be a positive integer.
If $\init{\dd}$ is in $GDT$ form, then
there is a plan of length less or equal than $n$ that solves $\pp$
if and only if the following quantified logic program is satisfiable:
\begin{equation}\label{solving:conditional}
\exists Occ_1 \forall Obs_1 \ldots \exists Occ_{n-1} \forall Obs_{n-1} \exists Occ_n \forall Open 
\big( \facts \cup \ttt{P} \cup (\ref{solving:generate})^{\leq}\cup(\ref{solving:flags})\cup(\ref{solving:obs}-\ref{solving:nonsensing}) \big)
\end{equation}
\end{theorem}
For $\dd_3$, that represents the problem $\pp_3$, and $n=3$,
the program (\ref{solving:conditional}) (adapted to $\dd_3$)
is satisfiable selecting first $\{occ(sense(occupied(1)))\}$ at $Occ_1$, 
then at $Occ_2$ selecting $\{occ(clean,2)\}$ for the subset $\{\} \subseteq Obs_1$ 
and $\{occ(go,2)\}$ for the subset $Obs_1 \subseteq Obs_1$, and 
finally at $Occ_3$ selecting $\{\}$ in all cases except for the 
subsets $Obs_1 \subseteq Obs_1$ and $\{\} \subseteq Obs_2$ that we select $\{occ(clean,3)\}$.
This assignment corresponds to plan $p_3$.

\section{Experiments}\label{sec:experiments}

We evaluate the performance of \qasp\ in conformant and conditional planning benchmarks.
%
We consider the problem \optt{} of computing a plan of optimal length.
To solve it we first run the solver for length 1, and
successively increment the length by 1 until an optimal plan is found.
The solving times of this procedure are usually dominated by the unsatisfiable runs.
To complement this, we also consider the problem \fl{}
of computing a plan of a fixed given length, for which we know that a solution exists.
%

We evaluate \qasp{} $1.0$ and combine it with $4$ differents QBF solvers:\footnote{%
We follow the selection of QBF solvers and preprocessors of~\citeANP{maysaf20a}~\citeyear{maysaf20a}.}
\caqe{} $4.0.1$~\cite{rabten15a}, \depqbf{} $6.03$~\cite{lonegl17a}, 
\qesto{} $1.0$~\cite{janmar15a} and \qute{} $1.1$~\cite{peslsz19a};
and either none or one of $3$ preprocessors: \bloqqer{} $v37$~(\textsc{b}; \citeNP{bilose11a}), 
\hqspre{} $1.4$~(\textsc{h}; \citeNP{wiremabe17a}), and 
\qratpre{} $2.0$~(\textsc{q}; \citeNP{lonegl19a});
for a total of $16$ configurations.
By \caqecirc{} we refer to the combination of \qasp{} with \caqe{} without preprocessor, and
by \caqe{b} to \qasp{} with \caqe{} and the preprocessor \bloqqer{}.
We proceed similarly for other QBF solvers and preprocessors.
%
%
%
We compare \qasp{} with the incomplete planner \cpasp{}~\cite{tusoba07a},
that translates a planning description into a normal logic program
that is fed to an ASP solver.
In the experiments we have used the ASP solver {\sc clingo} (version $5.5$)\footnote{\url{https://potassco.org/clingo}}
and evaluated its 6 basic configurations
(\textsc{crafty} (\textsc{c}), \textsc{frumpy} (\textsc{f}), \textsc{handy} (\textsc{h}), \textsc{jumpy} (\textsc{j}), \textsc{trendy} (\textsc{r}), and \textsc{tweety} (\textsc{t})).
We refer to \cpasp{} with \textsc{clingo} and configuration \textsc{crafty} by {\sclingo{c}}, 
and similarly with the others.

%
We use the benchmark set from~\citeN{tusoba07a},
but we have increased the size of the instances of some domains if they were too easy.
The conformant domains are six variants of
Bomb in the Toilet (Bt, Bmt, Btc, Bmtc, Btuc, Bmtuc), Domino and Ring.
We have also added a small variation of the Ring domain, called Ringu,
where the room of the agent is unknown,
%
and the planner \cpasp{} cannot find any plan due to its incompleteness.
The conditional domains are
four variations of Bomb in the Toilet with Sensing Actions (Bts1, Bts2, Bts3 and Bts4),
Domino,
Medical Problem (Med),
Ring and Sick.
All domains have 5 instances of increasing difficulty,
except Domino in conformant planning that has 6, and
Ring in conditional planning that has 4.
For the problem \fl{}, the fixed plan length is always the minimal plan length for \cpasp{}.
%
%
%
%

All experiments ran on an Intel Xeon 2.20GHz processor under Linux.
Each run was limited to 30 minutes runtime and 16 GB of memory.
We report the aggregated results per domain:
average runtime in seconds and number of timeouts in parentheses for \optt{},
next to the average runtime in seconds for \fl{}, for which there were very few timeouts.
To calculate the averages, we add 1800 seconds for every timeout.
%
In \ref{app:results},
we report these results for every solver and configuration,
and provide further details.
Here, we show and discuss the best configuration for each solver,
separately for conformant planning in Table~\ref{tab:conf-summary},
and for conditional planning in Table~\ref{tab:cond-summary}.
%

%
%
%

In conformant planning, looking at the \qasp{} configurations,
for the variations of Bt, \caqe{q} and \qesto{q} perform better than \depqbf{q} and \qute{h}.
Domino is solved very quickly by all solvers,
while in Ring and Ringu \caqe{q} clearly outperforms the others.
The planner \sclingo{j}, in the variations of Bt and Domino,
in \optt{} has a similar performance to the best \qasp{} solvers,
while in \fl{} it is much faster and solves the problems in less than a second.
In Ring, however, its performance is worse than that of \caqe{q}.
Finally, in Ringu for \optt{}, given the incompleteness of the system, it never manages to find a plan and always times out.
%
%
In conditional planning, \caqe{b} is the best for the variations of Bts,
while \sclingo{h} is better than the other solvers for \fl{} but worse for \optt{}.
In Domino, for \optt{}, the \qasp{} solvers perform better than \sclingo{h},
while for \fl{} only \qesto{b} matches its performance.
Finally, for Med, Ring, and Sick, all solvers yield similar times.
Summing up, we can conclude that \qasp{} with the right QBF solver and preprocessor compares well to \cpasp{},
except for the \fl{} problem in conformant planning,
while on the other hand it can solve problems, like Ringu,
that are out of reach for \cpasp{} due to its incompleteness.

\begin{table}[th]\footnotesize
\newcommand{\confsummaryavg}[6]{\bf{#1}&\hspace{5pt}\confsummaryavgI#2 &\hspace{5pt}\confsummaryavgII#3 &\hspace{5pt}\confsummaryavgIII#4 &\hspace{5pt}\confsummaryavgIV#5 &\hspace{5pt}\confsummaryavgV#6 \\}
\newcommand{\confsummaryavgI}[3]{#1\ (#3)\ &#2}
\newcommand{\confsummaryavgII}[3]{#1\ (#3)\ &#2}
\newcommand{\confsummaryavgIII}[3]{#1\ (#3)\ &#2}
\newcommand{\confsummaryavgIV}[3]{#1\ (#3)\ &#2}
\newcommand{\confsummaryavgV}[3]{#1\ (#3)\ &#2}

\(
\begin{array}{lr@{\ }rr@{\ }rr@{\ }rr@{\ }rr@{\ }r}\toprule
	\multirow[t]{2}{*}{} &\multicolumn{2}{c}{\multirow[t]{2}{*}{\qquad\sclingo{j}}} &\multicolumn{2}{c}{\multirow[t]{2}{*}{\qquad\caqe{q}}}& \multicolumn{2}{c}{\multirow[t]{2}{*}{\qquad\depqbf{q}}}& \multicolumn{2}{c}{\multirow[t]{2}{*}{\qquad\qesto{q}}} & \multicolumn{2}{c}{\multirow[t]{2}{*}{\qquad\qute{h}}} \\
	\midrule
	\confsummaryavg{Bt}{{43}{0}{0}     }{{46}{0}{0}  }{{319}{0}{0}    }{{57}{0}{0}}     {{363}{42}{1}}
	\confsummaryavg{Bmt}{{50}{0}{0}    }{{155}{2}{0} }{{367}{2}{1}    }{{76}{1}{0}}     {{382}{12}{1}}
	\confsummaryavg{Btc}{{258}{0}{0}   }{{413}{1}{1} }{{566}{2}{1}    }{{394}{1}{1}}    {{729}{38}{2}}
	\confsummaryavg{Bmtc}{{744}{0}{2}  }{{771}{9}{2} }{{1083}{9}{3}   }{{826}{9}{2}}    {{1082}{374}{3}}
	\confsummaryavg{Btuc}{{245}{0}{0}  }{{400}{1}{1} }{{538}{1}{1}    }{{391}{2}{1}}    {{731}{38}{2}}
	\confsummaryavg{Bmtuc}{{739}{0}{2} }{{780}{14}{2}}{{1085}{12}{3}  }{{813}{12}{2}}   {{1440}{378}{4}}
	\confsummaryavg{Domino}{{0}{0}{0}  }{{0}{0}{0}   }{{0}{0}{0}      }{{0}{0}{0}}      {{0}{0}{0}}
	\confsummaryavg{Ring}{{673}{416}{1}}{{281}{40}{0}}{{1120}{1109}{3}}{{790}{392}{2}}  {{1441}{1440}{4}}
	\confsummaryavg{Ringu}{{1800}{0}{5}}{{246}{50}{0}}{{1800}{1800}{5}}{{1800}{1800}{5}}{{1800}{1800}{5}}
	\midrule
	\confsummaryavg{Total}{{505}{46}{10}} {{343}{12}{6}} {{764}{326}{17}} {{571}{246}{13}} {{885}{458}{22}}
	\toprule
\end{array}
\)
\caption{Experimental results on conformant planning}
\label{tab:conf-summary}
\end{table}

\begin{table}[th]\footnotesize
\newcommand{\condsummaryavg}[6]{\bf{#1}&\hspace{5pt}\condsummaryavgI#2 &\hspace{5pt}\condsummaryavgII#3 &\hspace{5pt}\condsummaryavgIII#4 &\hspace{5pt}\condsummaryavgIV#5 &\hspace{5pt}\condsummaryavgV#6 \\}
\newcommand{\condsummaryavgI}[3]{#1\ (#3)\ &#2}
\newcommand{\condsummaryavgII}[3]{#1\ (#3)\ &#2}
\newcommand{\condsummaryavgIII}[3]{#1\ (#3)\ &#2}
\newcommand{\condsummaryavgIV}[3]{#1\ (#3)\ &#2}
\newcommand{\condsummaryavgV}[3]{#1\ (#3)\ &#2}

\(
\begin{array}{lr@{\ }rr@{\ }rr@{\ }rr@{\ }rr@{\ }r}\toprule
	\multirow[t]{2}{*}{} &\multicolumn{2}{c}{\multirow[t]{2}{*}{\qquad\sclingo{h}}} &\multicolumn{2}{c}{\multirow[t]{2}{*}{\qquad\caqe{b}}}& \multicolumn{2}{c}{\multirow[t]{2}{*}{\qquad\depqbf{b}}}& \multicolumn{2}{c}{\multirow[t]{2}{*}{\qquad\qesto{b}}} & \multicolumn{2}{c}{\multirow[t]{2}{*}{\qquad\qute{b}}} \\
	\midrule
	\condsummaryavg{Bts1}  {{795}{49}{2}}      {{14}{3}{0}}      {{364}{361}{1}}   {{102}{65}{0}}    {{365}{361}{1}}
	\condsummaryavg{Bts2}  {{787}{26}{2}}      {{14}{3}{0}}      {{369}{41}{1}}    {{370}{182}{1}}   {{376}{368}{1}}
	\condsummaryavg{Bts3}  {{802}{82}{2}}      {{15}{4}{0}}      {{372}{10}{1}}    {{371}{364}{1}}   {{380}{370}{1}}
	\condsummaryavg{Bts4}  {{833}{76}{2}}      {{17}{4}{0}}      {{368}{362}{1}}   {{379}{89}{1}}    {{390}{377}{1}}
	\condsummaryavg{Domino}{{906}{6}{2}}       {{362}{125}{1}}   {{361}{48}{1}}    {{363}{5}{1}}     {{361}{360}{1}}
	\condsummaryavg{Med}   {{0}{0}{0}}         {{3}{1}{0}}       {{3}{1}{0}}       {{3}{1}{0}}       {{3}{1}{0}}
	\condsummaryavg{Ring}  {{902}{900}{2}}     {{902}{640}{2}}   {{901}{900}{2}}   {{902}{900}{2}}   {{901}{900}{2}}
	\condsummaryavg{Sick}  {{0}{0}{0}}         {{1}{0}{0}}       {{1}{0}{0}}       {{1}{1}{0}}       {{1}{0}{0}}
	\midrule
	\condsummaryavg{Total}  {{628}{142}{12}}   {{166}{97}{3}}    {{342}{215}{7}}   {{311}{200}{6}}   {{347}{342}{7}}
	\toprule
\end{array}
\)
\captionsetup{skip=9pt}
\caption{Experimental results on conditional planning}
\label{tab:cond-summary}
\end{table}


\section{Related Work}

Conformant and conditional planning have already been addressed with ASP \cite{eifalepfpo03b,sotugemo05b,tusoba07a,tusogemo11a}.
\citeN{eifalepfpo03b} introduce the system \dlvk{} for planning with incomplete information.
It solves a conformant planning problem by first generating a potential plan and then verifying it,
no sensing actions are considered.
\citeANP{sotugemo05b}~\citeyear{sotugemo05b,tusoba07a,tusogemo11a} propose an approximation semantics for reasoning about
action and change in the presence of incomplete information and sensing actions.
This is then used for developing ASP-based conformant and conditional planners, like \cpasp, 
that are generally incomplete.

Closely related is SAT-based conformant planning:
\cplan\ \cite{cagita03a} is similar to \dlvk\ in identifying a potential plan before validating it.
{\sc compile-project-sat} \cite{palgef05a} uses a single call to a SAT-solver to compute a conformant plan.
Their validity check is doable in linear time, if the planning problem is encoded in deterministic
decomposable negation normal form. 
Unlike this,
{\sc QBFPlan} \cite{rintanen99b} maps the problem into QBF and
uses a QBF-solver as back-end.

A more recent use of ASP for computing conditional plans is proposed by \citeN{yanopaer17a}.
The planner deals with sensing actions and incomplete information;
it generates multiple sequential plans before combining them in a graph representing a conditional plan.
Cardinality constraints, defaults, and choices are used to represent
the execution of sensing actions, their effects, and branches in the final conditional plan.
In addition, the system computes sequential plans in parallel and
also avoids regenerating plans.

Assumption-based planning, as considered here, is due to \citeN{dabamc13a}.
In that work, the problem is solved by translating it into classical planning 
using an adaptation of the translation of \citeN{palgef09a}.
To the best of our knowledge, there exists no ASP-based planner for this type of problems.

There are a number of extensions of ASP to represent problems
whose complexity lays beyond NP in the polynomial hierarchy.
A comprehensive review was made by~\citeN{amritr19a},
that presents the approach that is closer to QASP,
named ASP with Quantifiers (ASP(Q)).
Like QASP, it introduces existential and universal quantifiers,
but they range over stable models of logic programs and not over atoms.
%
This quantification over stable models 
is very useful for knowledge representation. 
For example, it allows us to represent conformant planning problems
without the need of additional $\alpha{}$ atoms, 
using the following ASP(Q) program:
\begin{equation}\label{related:aspq:conformant}
\exists^{st} \big( \facts \cup (\ref{solving:generate})\big) \ \forall^{st} \ttx{\dd_\initx} \ 
\exists^{st} \big( \ttx{\dd_\dynx} \cup \ttx{\dd_\finalx} \big)
\end{equation}
where $\exists^{st}$ and $\forall^{st}$ are existential and universal stable model quantifiers, respectively.
The program is coherent (or satisfiable, in our terms) if
there is some stable model $M_1$ of $\facts \cup (\ref{solving:generate})$ 
such that for all stable models $M_2$ of $M_1 \cup \ttx{\dd_\initx}$ 
there is some stable model of $M_2  \cup \ttx{\dd_\dynx} \cup \ttx{\dd_\finalx}$.
Stable models $M_1$ correspond to possible plans.
They are extended in $M_2$ by atoms representing initial states, 
that are used in $M_2  \cup \ttx{\dd_\dynx} \cup \ttx{\dd_\finalx}$ 
to check if the plans achieve the goal starting from \emph{all} those initial states.
Assumption-based planning can be represented in a similar way, 
while for conditional planning we have not been able to come 
up with any formulation that does not use additional $\alpha$ atoms.
%

As part of this work, we have developed translations between QASP and ASP(Q).
%
We leave their formal specification to \ref{app:aspq}, and
illustrate them here for conformant planning.
From ASP(Q) to QASP, 
we assume that $\ttx{\dd_\initx}$ is in \emph{GDT} form,
and otherwise we translate it into this form.
Then, the translation of an
ASP(Q) program of the form (\ref{related:aspq:conformant}) yields a QLP
that is essentially the same as (\ref{solving:conformant}), 
except for some renaming of the additional $\alpha$ atoms
and some irrelevant changes in the prefix. 
%
%
In the other direction, the QLP program 
(\ref{solving:conformant}) is translated to the ASP(Q) program 
\begin{equation}\label{related:aspq:qasp}
\exists^{st} P_0 \forall^{st} P_1 \exists^{st}
\big( \facts \cup \ttt{\dd} \cup (\ref{solving:generate})\cup(\ref{solving:flags}) \cup O\big)
\end{equation}
where $P_0$ is $\{ \{ p' \}\leftarrow \mid p \in Occ \}$, 
      $P_1$ is $\{ \{ p' \}\leftarrow \mid p \in \open{} \}$, and
$O$ contains the set of rules
$\bot \leftarrow p, \naf{p'} $ and $\bot \leftarrow \naf{p}, p' $ 
for every $p \in Occ \cup \open{}$.
Programs $P_0$ and $P_1$ guess the values of the atoms of the prefix using additional atoms $p'$, 
and the constraints in $O$ match those guesses to the corresponding original atoms.

\section{Conclusion}

We defined a general ASP language
to represent a wide range of planning \emph{problems}:
classical, conformant, with assumptions,
and conditional with sensing actions.
We then defined a quantified extension of ASP, viz.\ QASP,
to represent the \emph{solutions} to those planning problems.
Finally,
we implemented and evaluated a QASP solver,
available at \url{potassco.org}, 
to \emph{compute} the solutions to those planning problems.
Our focus lays on the generality of the language and the tackled problems;
on the formal foundations of the approach, by relating it to simple transition functions; and
on having a baseline implementation, whose performance we expect to improve further in the future.


\bibliographystyle{acmtrans}

\appendix
\clearpage
\section{Results}\label{app:results}

In this section, we provide some additional information about the experiments,
and present the results for every solver and configuration.

We note that in conditional planning, \cpasp{} receives as input an additional width
parameter that determines the minimum number of branches of a conditional plan.
\qasp{} receives no information of that kind, and in this regard it solves a more
difficult task than \cpasp{}.
On the other hand, that width parameter allows \cpasp{} to return the complete conditional plan,
while \qasp{} only returns the assignment to the first action of the plan.
%
%
Observe also that given that \cpasp{} is incomplete, in some instances \qasp{}
can find plans for \optt{} whose length is smaller than the fixed length of \fl{}.
This explains why in some few cases \qasp{} solves faster the \optt{} problem than \fl{}.

In the experiments, the times for grounding and for the translators \lpnormal{} and \lpsat{} are negligible,
while, in the \fl{} problem, preprocessing takes on average $5$ and $2$ seconds in conformant and conditional planning,
respectively. We expect these times to be similar for the \optt{} problem.
The reported times are dominated by \sclingo{}'s solving time in \cpasp{}, and by the QBF solvers in \qasp.
But the usage of preprocessors is fundamental for the performance of \qasp.
For example, in conformant planning, \caqe{} alone results in $18$ timeouts,
while with \qratpre{} they go down to only $6$.



\begin{table}[th]\footnotesize
\newcommand{\confclingoavg}[7]{\bf{#1}&\hspace{0pt}\confclingoavgI#5 &\hspace{0pt}\confclingoavgII#4 &\hspace{0pt}\confclingoavgIII#7 &\hspace{0pt}\confclingoavgIV#6 &\hspace{0pt}\confclingoavgV#3 &\hspace{0pt}\confclingoavgVI#2 \\}
\newcommand{\confclingoavgI}[3]{#1\ (#3)&#2}
\newcommand{\confclingoavgII}[3]{#1\ (#3)&#2}
\newcommand{\confclingoavgIII}[3]{#1\ (#3)&#2}
\newcommand{\confclingoavgIV}[3]{#1\ (#3)&#2}
\newcommand{\confclingoavgV}[3]{#1\ (#3)&#2}
\newcommand{\confclingoavgVI}[3]{#1\ (#3)&#2}

\(
\begin{array}{lr@{\ }rr@{\ }rr@{\ }rr@{\ }rr@{\ }rr@{\ }r}\toprule
	\multirow[t]{2}{*}{} &\multicolumn{2}{c}{\multirow[t]{2}{*}{\quad\sclingo{c}}} &\multicolumn{2}{c}{\multirow[t]{2}{*}{\quad\sclingo{f}}}& \multicolumn{2}{c}{\multirow[t]{2}{*}{\quad\sclingo{h}}}& \multicolumn{2}{c}{\multirow[t]{2}{*}{\quad\sclingo{j}}} & \multicolumn{2}{c}{\multirow[t]{2}{*}{\quad\sclingo{r}}} & \multicolumn{2}{c}{\multirow[t]{2}{*}{\quad\sclingo{t}}} \\
	\midrule
	\confclingoavg{Bt}{{48}{0}{0}}{{363}{0}{1}}{{176}{0}{0}}{{52}{0}{0}}{{43}{0}{0}}{{362}{0}{1}}
	\confclingoavg{Bmt}{{35}{0}{0}}{{363}{0}{1}}{{163}{0}{0}}{{49}{0}{0}}{{50}{0}{0}}{{238}{0}{0}}
	\confclingoavg{Btc}{{282}{0}{0}}{{380}{0}{1}}{{395}{0}{1}}{{236}{0}{0}}{{258}{0}{0}}{{375}{0}{1}}
	\confclingoavg{Bmtc}{{739}{0}{2}}{{773}{0}{2}}{{998}{0}{2}}{{746}{211}{2}}{{744}{0}{2}}{{743}{0}{2}}
	\confclingoavg{Btuc}{{380}{0}{1}}{{382}{0}{1}}{{389}{0}{1}}{{367}{0}{1}}{{245}{0}{0}}{{372}{0}{1}}
	\confclingoavg{Bmtuc}{{742}{0}{2}}{{778}{0}{2}}{{938}{0}{2}}{{749}{0}{2}}{{739}{0}{2}}{{741}{0}{2}}
	\confclingoavg{Domino}{{0}{0}{0}}{{0}{0}{0}}{{0}{0}{0}}{{0}{0}{0}}{{0}{0}{0}}{{0}{0}{0}}
	\confclingoavg{Ring}{{803}{364}{2}}{{645}{204}{1}}{{1102}{777}{3}}{{774}{506}{2}}{{673}{416}{1}}{{538}{96}{1}}
	\confclingoavg{Ringu}{{1800}{0}{5}}{{1800}{0}{5}}{{1800}{0}{5}}{{1800}{0}{5}}{{1800}{0}{5}}{{1800}{0}{5}}
	\midrule
	\confclingoavg{Total}{{536}{40}{12}}{{609}{22}{14}}{{662}{86}{14}}{{530}{79}{12}}{{505}{46}{10}}{{574}{10}{13}}
	\toprule
\end{array}
\)
\caption{Conformant planning with \sclingo{}}
\label{tab:conf-clingo-avg}
\end{table}

\begin{table}[th]\footnotesize
\newcommand{\confcaqeavg}[5]{\bf{#1}&\hspace{5pt}\confcaqeavgI#2 &\hspace{5pt}\confcaqeavgII#3 &\hspace{5pt}\confcaqeavgIII#5 &\hspace{5pt}\confcaqeavgIV#4 \\}
\newcommand{\confcaqeavgI}[3]{#1\ (#3)\ &#2}
\newcommand{\confcaqeavgII}[3]{#1\ (#3)\ &#2}
\newcommand{\confcaqeavgIII}[3]{#1\ (#3)\ &#2}
\newcommand{\confcaqeavgIV}[3]{#1\ (#3)\ &#2}

\(
\begin{array}{lr@{\ }rr@{\ }rr@{\ }rr@{\ }r}\toprule
    \multirow[t]{2}{*}{} &\multicolumn{2}{c}{\multirow[t]{2}{*}{\qquad \caqecirc }} &\multicolumn{2}{c}{\multirow[t]{2}{*}{\qquad\caqe{b}}}& \multicolumn{2}{c}{\multirow[t]{2}{*}{\qquad\caqe{h}}}& \multicolumn{2}{c}{\multirow[t]{2}{*}{\qquad\caqe{q}}} \\
	\midrule
	\confcaqeavg{Bt}{{69}{3}{0}}{{377}{1}{0}}{{46}{0}{0}}{{364}{1}{1}}
	\confcaqeavg{Bmt}{{44}{3}{0}}{{149}{4}{0}}{{155}{2}{0}}{{373}{11}{1}}
	\confcaqeavg{Btc}{{440}{3}{1}}{{487}{3}{1}}{{413}{1}{1}}{{577}{13}{1}}
	\confcaqeavg{Bmtc}{{814}{7}{2}}{{777}{7}{2}}{{771}{9}{2}}{{1083}{32}{3}}
	\confcaqeavg{Btuc}{{476}{3}{1}}{{475}{3}{1}}{{400}{1}{1}}{{606}{9}{1}}
	\confcaqeavg{Bmtuc}{{849}{14}{2}}{{789}{15}{2}}{{780}{14}{2}}{{1260}{25}{3}}
	\confcaqeavg{Domino}{{1500}{1500}{5}}{{0}{0}{0}}{{0}{0}{0}}{{0}{0}{0}}
	\confcaqeavg{Ring}{{1104}{1097}{3}}{{219}{30}{0}}{{281}{40}{0}}{{837}{793}{2}}
	\confcaqeavg{Ringu}{{1630}{1643}{4}}{{1018}{969}{2}}{{246}{45}{0}}{{1506}{1505}{4}}
	\midrule
	\confcaqeavg{Total}{{769}{474}{18}}{{476}{114}{8}}{{343}{12}{6}}{{734}{265}{16}}
	\toprule
\end{array}
\)
\caption{Conformant planning with \caqe{}}
\label{tab:conf-caqe-avg}
\end{table}

\begin{table}[th]\footnotesize
\newcommand{\confdepqbfavg}[5]{\bf{#1}&\hspace{5pt}\confdepqbfavgI#2 &\hspace{5pt}\confdepqbfavgII#3 &\hspace{5pt}\confdepqbfavgIII#5 &\hspace{5pt}\confdepqbfavgIV#4 \\}
\newcommand{\confdepqbfavgI}[3]{#1\ (#3)\ &#2}
\newcommand{\confdepqbfavgII}[3]{#1\ (#3)\ &#2}
\newcommand{\confdepqbfavgIII}[3]{#1\ (#3)\ &#2}
\newcommand{\confdepqbfavgIV}[3]{#1\ (#3)\ &#2}

\(
\begin{array}{lr@{\ }rr@{\ }rr@{\ }rr@{\ }r}\toprule
	\multirow[t]{2}{*}{} &\multicolumn{2}{c}{\multirow[t]{2}{*}{\qquad\depqbfcirc}} &\multicolumn{2}{c}{\multirow[t]{2}{*}{\qquad\depqbf{b}}}& \multicolumn{2}{c}{\multirow[t]{2}{*}{\qquad\depqbf{h}}}& \multicolumn{2}{c}{\multirow[t]{2}{*}{\qquad\depqbf{q}}} \\
	\midrule
	\confdepqbfavg{Bt}{{740}{0}{2}}{{744}{1}{2}}{{319}{0}{0}}{{363}{1}{1}}
	\confdepqbfavg{Bmt}{{974}{2}{2}}{{1089}{3}{3}}{{367}{2}{1}}{{372}{11}{1}}
	\confdepqbfavg{Btc}{{1111}{3}{3}}{{1084}{3}{3}}{{566}{2}{1}}{{729}{10}{2}}
	\confdepqbfavg{Bmtc}{{1494}{5}{4}}{{1442}{7}{4}}{{1083}{9}{3}}{{1082}{20}{3}}
	\confdepqbfavg{Btuc}{{1102}{7}{3}}{{1088}{3}{3}}{{538}{1}{1}}{{730}{11}{2}}
	\confdepqbfavg{Bmtuc}{{1479}{82}{4}}{{1443}{8}{4}}{{1085}{12}{3}}{{1440}{39}{4}}
	\confdepqbfavg{Domino}{{0}{0}{0}}{{1}{0}{0}}{{0}{0}{0}}{{0}{0}{0}}
	\confdepqbfavg{Ring}{{1800}{1800}{5}}{{1708}{1468}{4}}{{1120}{1109}{3}}{{1108}{736}{3}}
	\confdepqbfavg{Ringu}{{1800}{1800}{5}}{{1800}{1800}{5}}{{1800}{1800}{5}}{{1800}{1800}{5}}
	\midrule
	\confdepqbfavg{Total}{{1166}{411}{28}}{{1155}{365}{28}}{{764}{326}{17}}{{847}{292}{21}}
	\toprule
\end{array}
\)
\caption{ Conformant planning with \depqbf{}}
\label{tab:conf-depqbf-avg}
\end{table}

\begin{table}[th]\footnotesize
\newcommand{\confqestoavg}[5]{\bf{#1}&\hspace{5pt}\confqestoavgI#2 &\hspace{5pt}\confqestoavgII#3 &\hspace{5pt}\confqestoavgIII#5 &\hspace{5pt}\confqestoavgIV#4 \\}
\newcommand{\confqestoavgI}[3]{#1\ (#3)\ &#2}
\newcommand{\confqestoavgII}[3]{#1\ (#3)\ &#2}
\newcommand{\confqestoavgIII}[3]{#1\ (#3)\ &#2}
\newcommand{\confqestoavgIV}[3]{#1\ (#3)\ &#2}

\(
\begin{array}{lr@{\ }rr@{\ }rr@{\ }rr@{\ }r}\toprule
	\multirow[t]{2}{*}{} &\multicolumn{2}{c}{\multirow[t]{2}{*}{\qquad\qestocirc}} &\multicolumn{2}{c}{\multirow[t]{2}{*}{\qquad\qesto{b}}}& \multicolumn{2}{c}{\multirow[t]{2}{*}{\qquad\qesto{h}}}& \multicolumn{2}{c}{\multirow[t]{2}{*}{\qquad\qesto{q}}} \\
	\midrule
	\confqestoavg{Bt}{{725}{3}{2}}{{458}{1}{1}}{{57}{0}{0}}{{363}{1}{1}}
	\confqestoavg{Bmt}{{1081}{9}{3}}{{443}{4}{1}}{{76}{1}{0}}{{372}{13}{1}}
	\confqestoavg{Btc}{{1081}{7}{3}}{{777}{6}{2}}{{394}{1}{1}}{{729}{9}{2}}
	\confqestoavg{Bmtc}{{1440}{20}{4}}{{1104}{23}{3}}{{826}{9}{2}}{{1082}{23}{3}}
	\confqestoavg{Btuc}{{1082}{15}{3}}{{736}{6}{1}}{{391}{2}{1}}{{734}{9}{2}}
	\confqestoavg{Bmtuc}{{1441}{359}{4}}{{1094}{27}{3}}{{813}{12}{2}}{{1286}{26}{3}}
	\confqestoavg{Domino}{{1319}{1339}{4}}{{0}{0}{0}}{{0}{0}{0}}{{0}{0}{0}}
	\confqestoavg{Ring}{{1105}{1096}{3}}{{499}{30}{1}}{{790}{392}{2}}{{782}{752}{2}}
	\confqestoavg{Ringu}{{1800}{1800}{5}}{{1444}{1440}{4}}{{1800}{1800}{5}}{{1800}{1800}{5}}
	\midrule
	\confqestoavg{Total}{{1230}{516}{31}}{{728}{170}{16}}{{571}{246}{13}}{{794}{292}{19}}
	\toprule
\end{array}
\)
\caption{Conformant planning with \qesto{}}
\label{tab:conf-qesto-avg}
\end{table}

\begin{table}[th]\footnotesize
\newcommand{\confquteavg}[5]{\bf{#1}&\hspace{5pt}\confquteavgI#2 &\hspace{5pt}\confquteavgII#3 &\hspace{5pt}\confquteavgIII#5 &\hspace{5pt}\confquteavgIV#4 \\}
\newcommand{\confquteavgI}[3]{#1\ (#3)\ &#2}
\newcommand{\confquteavgII}[3]{#1\ (#3)\ &#2}
\newcommand{\confquteavgIII}[3]{#1\ (#3)\ &#2}
\newcommand{\confquteavgIV}[3]{#1\ (#3)\ &#2}

\(
\begin{array}{lr@{\ }rr@{\ }rr@{\ }rr@{\ }r}\toprule
	\multirow[t]{2}{*}{} &\multicolumn{2}{c}{\multirow[t]{2}{*}{\qquad\qutecirc}} &\multicolumn{2}{c}{\multirow[t]{2}{*}{\qquad\qute{b}}}& \multicolumn{2}{c}{\multirow[t]{2}{*}{\qquad\qute{h}}}& \multicolumn{2}{c}{\multirow[t]{2}{*}{\qquad\qute{q}}} \\
	\midrule
	\confquteavg{Bt}{{864}{3}{2}}{{1081}{1}{3}}{{852}{0}{2}}{{363}{42}{1}}
	\confquteavg{Bmt}{{1095}{9}{3}}{{1180}{3}{3}}{{477}{1}{1}}{{382}{12}{1}}
	\confquteavg{Btc}{{1095}{16}{3}}{{1093}{3}{3}}{{1093}{1}{3}}{{729}{38}{2}}
	\confquteavg{Bmtc}{{1453}{27}{4}}{{1444}{8}{4}}{{1484}{9}{4}}{{1082}{374}{3}}
	\confquteavg{Btuc}{{1111}{55}{3}}{{1101}{3}{3}}{{1116}{1}{3}}{{731}{38}{2}}
	\confquteavg{Bmtuc}{{1501}{557}{4}}{{1745}{24}{4}}{{1483}{12}{4}}{{1440}{378}{4}}
	\confquteavg{Domino}{{1238}{1239}{4}}{{0}{0}{0}}{{0}{0}{0}}{{0}{0}{0}}
	\confquteavg{Ring}{{1800}{1449}{5}}{{1702}{1441}{4}}{{1460}{1208}{4}}{{1441}{1440}{4}}
	\confquteavg{Ringu}{{1800}{1800}{5}}{{1800}{1599}{5}}{{1800}{1800}{5}}{{1800}{1800}{5}}
	\midrule
	\confquteavg{Total}{{1328}{572}{33}}{{1238}{342}{29}}{{1085}{336}{26}}{{885}{458}{22}}
	\toprule
\end{array}
\)
\caption{Conformant planning with \qute{}}
\label{tab:conf-qute-avg}
\end{table}

\begin{table}[th]\footnotesize
\newcommand{\condclingoavg}[7]{\bf{#1}&\hspace{0pt}\condclingoavgI#5 &\hspace{0pt}\condclingoavgII#4 &\hspace{0pt}\condclingoavgIII#7 &\hspace{0pt}\condclingoavgIV#6 &\hspace{0pt}\condclingoavgV#3 &\hspace{0pt}\condclingoavgVI#2 \\}
\newcommand{\condclingoavgI}[3]{#1\ (#3)&#2}
\newcommand{\condclingoavgII}[3]{#1\ (#3)&#2}
\newcommand{\condclingoavgIII}[3]{#1\ (#3)&#2}
\newcommand{\condclingoavgIV}[3]{#1\ (#3)&#2}
\newcommand{\condclingoavgV}[3]{#1\ (#3)&#2}
\newcommand{\condclingoavgVI}[3]{#1\ (#3)&#2}

\(
\begin{array}{lr@{\ }rr@{\ }rr@{\ }rr@{\ }rr@{\ }rr@{\ }r}\toprule
	\multirow[t]{2}{*}{} &\multicolumn{2}{c}{\multirow[t]{2}{*}{\quad\sclingo{c}}} &\multicolumn{2}{c}{\multirow[t]{2}{*}{\quad\sclingo{f}}}& \multicolumn{2}{c}{\multirow[t]{2}{*}{\quad\sclingo{h}}}& \multicolumn{2}{c}{\multirow[t]{2}{*}{\quad\sclingo{j}}} & \multicolumn{2}{c}{\multirow[t]{2}{*}{\quad\sclingo{r}}} & \multicolumn{2}{c}{\multirow[t]{2}{*}{\quad\sclingo{t}}} \\
	\midrule
	\condclingoavg{Bts1}{{927}{370}{2}}{{916}{28}{2}}{{1080}{77}{3}}{{1041}{121}{2}}{{989}{51}{2}}            {{795}{49}{2}}
	\condclingoavg{Bts2}{{1080}{71}{3}}{{920}{66}{2}}{{1080}{361}{3}}{{1080}{60}{3}}{{986}{30}{2}}            {{787}{26}{2}}
	\condclingoavg{Bts3}{{1080}{253}{3}}{{944}{68}{2}}{{1071}{215}{2}}{{993}{85}{2}}{{1012}{50}{2}}           {{802}{82}{2}}
	\condclingoavg{Bts4}{{1080}{173}{3}}{{904}{70}{2}}{{1080}{115}{3}}{{1060}{69}{2}}{{882}{102}{2}}          {{833}{76}{2}}
	\condclingoavg{Domino}{{759}{18}{2}}{{786}{7}{2}}{{885}{59}{2}}{{872}{15}{2}}{{771}{13}{2}}               {{906}{6}{2}}
	\condclingoavg{Med}{{0}{0}{0}}{{0}{0}{0}}{{0}{0}{0}}{{0}{0}{0}}{{0}{0}{0}}                                {{0}{0}{0}}
	\condclingoavg{Ring}{{902}{900}{2}}{{902}{900}{2}}{{905}{900}{2}}{{902}{900}{2}}{{903}{900}{2}}           {{902}{900}{2}}
	\condclingoavg{Sick}{{0}{0}{0}}{{0}{0}{0}}{{0}{0}{0}}{{0}{0}{0}}{{0}{0}{0}}                               {{0}{0}{0}}
	\midrule
	\condclingoavg{Total}{{728}{223}{15}}{{671}{142}{12}}{{762}{215}{15}}{{743}{156}{13}}{{692}{143}{12}}{{628}{142}{12}}
	\toprule
\end{array}
\)
\caption{Conditional planning with \sclingo{}}
\label{tab:cond-clingo-avg}
\end{table}

\begin{table}[th]\footnotesize
\newcommand{\condcaqeavg}[5]{\bf{#1}&\hspace{5pt}\condcaqeavgI#2 &\hspace{5pt}\condcaqeavgII#3 &\hspace{5pt}\condcaqeavgIII#5 &\hspace{5pt}\condcaqeavgIV#4 \\}
\newcommand{\condcaqeavgI}[3]{#1\ (#3)\ &#2}
\newcommand{\condcaqeavgII}[3]{#1\ (#3)\ &#2}
\newcommand{\condcaqeavgIII}[3]{#1\ (#3)\ &#2}
\newcommand{\condcaqeavgIV}[3]{#1\ (#3)\ &#2}

\(
\begin{array}{lr@{\ }rr@{\ }rr@{\ }rr@{\ }r}\toprule
	\multirow[t]{2}{*}{} &\multicolumn{2}{c}{\multirow[t]{2}{*}{\qquad\caqecirc}} &\multicolumn{2}{c}{\multirow[t]{2}{*}{\qquad\caqe{b}}}& \multicolumn{2}{c}{\multirow[t]{2}{*}{\qquad\caqe{h}}}& \multicolumn{2}{c}{\multirow[t]{2}{*}{\qquad\caqe{q}}} \\
	\midrule
	\condcaqeavg{Bts1}{{370}{368}{1}}{{14}{3}{0}}{{366}{363}{1}}{{16}{3}{0}}
	\condcaqeavg{Bts2}{{378}{376}{1}}{{14}{3}{0}}{{106}{54}{0}}{{15}{4}{0}}
	\condcaqeavg{Bts3}{{372}{379}{1}}{{15}{4}{0}}{{135}{65}{0}}{{19}{7}{0}}
	\condcaqeavg{Bts4}{{374}{384}{1}}{{17}{4}{0}}{{107}{53}{0}}{{25}{8}{0}}
	\condcaqeavg{Domino}{{372}{365}{1}}{{362}{125}{1}}{{223}{113}{0}}{{366}{297}{1}}
	\condcaqeavg{Med}{{1095}{1440}{3}}{{3}{1}{0}}{{4}{2}{0}}{{2}{0}{0}}
	\condcaqeavg{Ring}{{1350}{1350}{3}}{{902}{640}{2}}{{904}{901}{2}}{{901}{901}{2}}
	\condcaqeavg{Sick}{{729}{729}{2}}{{1}{0}{0}}{{64}{63}{0}}{{1}{1}{0}}
	\midrule
	\condcaqeavg{Total}{{630}{673}{13}}{{166}{97}{3}}{{238}{201}{3}}{{168}{152}{3}}
	\toprule
\end{array}
\)
\caption{Conditional planning with \caqe{}}
\label{tab:cond-caqe-avg}
\end{table}

\begin{table}[th]\footnotesize
\newcommand{\conddepqbfavg}[5]{\bf{#1}&\hspace{5pt}\conddepqbfavgI#2 &\hspace{5pt}\conddepqbfavgII#3 &\hspace{5pt}\conddepqbfavgIII#5 &\hspace{5pt}\conddepqbfavgIV#4 \\}
\newcommand{\conddepqbfavgI}[3]{#1\ (#3)\ &#2}
\newcommand{\conddepqbfavgII}[3]{#1\ (#3)\ &#2}
\newcommand{\conddepqbfavgIII}[3]{#1\ (#3)\ &#2}
\newcommand{\conddepqbfavgIV}[3]{#1\ (#3)\ &#2}

\(
\begin{array}{lr@{\ }rr@{\ }rr@{\ }rr@{\ }r}\toprule
	\multirow[t]{2}{*}{} &\multicolumn{2}{c}{\multirow[t]{2}{*}{\qquad\depqbfcirc}} &\multicolumn{2}{c}{\multirow[t]{2}{*}{\qquad\depqbf{b}}}& \multicolumn{2}{c}{\multirow[t]{2}{*}{\qquad\depqbf{h}}}& \multicolumn{2}{c}{\multirow[t]{2}{*}{\qquad\depqbf{q}}} \\
	\midrule
	\conddepqbfavg{Bts1}{{722}{720}{2}}{{364}{361}{1}}{{721}{506}{2}}{{396}{366}{1}}
	\conddepqbfavg{Bts2}{{724}{720}{2}}{{369}{41}{1}}{{723}{720}{2}}{{489}{362}{1}}
	\conddepqbfavg{Bts3}{{729}{720}{2}}{{372}{10}{1}}{{726}{720}{2}}{{721}{361}{2}}
	\conddepqbfavg{Bts4}{{745}{720}{2}}{{368}{362}{1}}{{731}{720}{2}}{{722}{654}{2}}
	\conddepqbfavg{Domino}{{474}{158}{1}}{{361}{48}{1}}{{376}{360}{1}}{{366}{7}{1}}
	\conddepqbfavg{Med}{{2}{0}{0}}{{3}{1}{0}}{{1}{0}{0}}{{2}{0}{0}}
	\conddepqbfavg{Ring}{{1340}{1350}{2}}{{901}{900}{2}}{{902}{902}{2}}{{902}{901}{2}}
	\conddepqbfavg{Sick}{{55}{48}{0}}{{1}{0}{0}}{{0}{0}{0}}{{1}{0}{0}}
	\midrule
	\conddepqbfavg{Total}{{598}{554}{11}}{{342}{215}{7}}{{522}{491}{11}}{{449}{331}{9}}
	\toprule
\end{array}
\)
\caption{Conditional planning with \depqbf{}}
\label{tab:cond-depqbf-avg}
\end{table}

\begin{table}[th]\footnotesize
\newcommand{\condqestoavg}[5]{\bf{#1}&\hspace{5pt}\condqestoavgI#2 &\hspace{5pt}\condqestoavgII#3 &\hspace{5pt}\condqestoavgIII#5 &\hspace{5pt}\condqestoavgIV#4 \\}
\newcommand{\condqestoavgI}[3]{#1\ (#3)\ &#2}
\newcommand{\condqestoavgII}[3]{#1\ (#3)\ &#2}
\newcommand{\condqestoavgIII}[3]{#1\ (#3)\ &#2}
\newcommand{\condqestoavgIV}[3]{#1\ (#3)\ &#2}

\(
\begin{array}{lr@{\ }rr@{\ }rr@{\ }rr@{\ }r}\toprule
	\multirow[t]{2}{*}{} &\multicolumn{2}{c}{\multirow[t]{2}{*}{\qquad\qestocirc}} &\multicolumn{2}{c}{\multirow[t]{2}{*}{\qquad\qesto{b}}}& \multicolumn{2}{c}{\multirow[t]{2}{*}{\qquad\qesto{h}}}& \multicolumn{2}{c}{\multirow[t]{2}{*}{\qquad\qesto{q}}} \\
	\midrule
	\condqestoavg{Bts1}{{380}{373}{1}}{{102}{65}{0}}{{400}{391}{1}}{{369}{42}{1}}
	\condqestoavg{Bts2}{{384}{379}{1}}{{370}{182}{1}}{{390}{378}{1}}{{720}{720}{2}}
	\condqestoavg{Bts3}{{390}{379}{1}}{{371}{364}{1}}{{396}{377}{1}}{{722}{362}{2}}
	\condqestoavg{Bts4}{{392}{382}{1}}{{379}{89}{1}}{{392}{379}{1}}{{721}{720}{2}}
	\condqestoavg{Domino}{{720}{413}{2}}{{363}{5}{1}}{{530}{481}{1}}{{366}{360}{1}}
	\condqestoavg{Med}{{1100}{1440}{3}}{{3}{1}{0}}{{5}{6}{0}}{{1}{0}{0}}
	\condqestoavg{Ring}{{1350}{1350}{3}}{{902}{900}{2}}{{905}{902}{2}}{{902}{900}{2}}
	\condqestoavg{Sick}{{739}{736}{2}}{{1}{1}{0}}{{11}{11}{0}}{{1}{0}{0}}
	\midrule
	\condqestoavg{Total}{{681}{681}{14}}{{311}{200}{6}}{{378}{365}{7}}{{475}{388}{10}}
	\toprule
\end{array}
\)
\caption{Conditional planning with \qesto{}}
\label{tab:cond-qesto-avg}
\end{table}

\begin{table}[th]\footnotesize
\newcommand{\condquteavg}[5]{\bf{#1}&\hspace{5pt}\condquteavgI#2 &\hspace{5pt}\condquteavgII#3 &\hspace{5pt}\condquteavgIII#5 &\hspace{5pt}\condquteavgIV#4 \\}
\newcommand{\condquteavgI}[3]{#1\ (#3)\ &#2}
\newcommand{\condquteavgII}[3]{#1\ (#3)\ &#2}
\newcommand{\condquteavgIII}[3]{#1\ (#3)\ &#2}
\newcommand{\condquteavgIV}[3]{#1\ (#3)\ &#2}

\(
\begin{array}{lr@{\ }rr@{\ }rr@{\ }rr@{\ }r}\toprule
	\multirow[t]{2}{*}{} &\multicolumn{2}{c}{\multirow[t]{2}{*}{\qquad\qutecirc}} &\multicolumn{2}{c}{\multirow[t]{2}{*}{\qquad\qute{b}}}& \multicolumn{2}{c}{\multirow[t]{2}{*}{\qquad\qute{h}}}& \multicolumn{2}{c}{\multirow[t]{2}{*}{\qquad\qute{q}}} \\
	\midrule
	\condquteavg{Bts1}{{770}{760}{2}}{{365}{361}{1}}{{751}{726}{2}}{{724}{722}{2}}
	\condquteavg{Bts2}{{1080}{1080}{3}}{{376}{368}{1}}{{1080}{1080}{3}}{{724}{721}{2}}
	\condquteavg{Bts3}{{896}{901}{2}}{{380}{370}{1}}{{1080}{944}{3}}{{783}{777}{2}}
	\condquteavg{Bts4}{{966}{916}{2}}{{390}{377}{1}}{{1081}{1080}{3}}{{731}{726}{2}}
	\condquteavg{Domino}{{722}{408}{2}}{{361}{360}{1}}{{606}{451}{1}}{{720}{720}{2}}
	\condquteavg{Med}{{1104}{1440}{3}}{{3}{1}{0}}{{27}{24}{0}}{{5}{3}{0}}
	\condquteavg{Ring}{{1350}{1350}{3}}{{901}{900}{2}}{{917}{911}{2}}{{907}{903}{2}}
	\condquteavg{Sick}{{782}{776}{2}}{{1}{0}{0}}{{6}{6}{0}}{{1}{1}{0}}
	\midrule
	\condquteavg{Total}{{958}{953}{19}}{{347}{342}{7}}{{693}{652}{14}}{{574}{571}{12}}
	\toprule
\end{array}
\)
\caption{Conditional planning with \qute{}}
\label{tab:cond-qute-avg}
\end{table}

\clearpage
\section{ASP(Q) and QASP}\label{app:aspq}

An ASP(Q) program $\Pi$ is an expression of the form:
\begin{equation}
\label{aspq:program}
\square_0 P_0 \square_1 P_1 \cdots \square_n P_n : C,
\end{equation}
where $n \geq 0$, 
$P_0$, \ldots, $P_n$ are logic programs,
every $\square_i$ is either $\exists^{st}$ or $\forall^{st}$,
and $C$ is a normal stratified program.
In~\cite{amritr19a}, the logic programs $P_i$ can be disjunctive.
Here, for simplicity, we restrict ourselves to logic programs as defined in the Background section.
Symbols $\exists^{st}$ and $\forall^{st}$ are called 
existential and universal answer set quantifiers, respectively. 
We say that the program $P_i$ is existentially (universally, respectively) quantified if
$\square_i$ is $\exists^{st}$ ($\forall^{st}$, respectively).
%
%
%
Let $X$ and $Y$ be sets of atoms such that $X \subseteq Y$,
by $\fixfact{X}{Y}$ we denote the set of rules
$\{ x \leftarrow \mid x \in X \} \cup \{ \leftarrow x \mid x \in Y \setminus X \}$.
Let $\atoms{P}$ denote the set of atoms occurring in a logic program $P$.
Given an ASP(Q) program $\Pi$, a logic program $P$, and a set of atoms $X$, 
we denote by $\Pi_{P,X}$ the program of the form (\ref{aspq:program}), 
where $P_0$ is replaced by $P_0 \cup \fixfact{X}{\atoms{P}}$, that is, 
$\Pi_{P,X} = \square_0 (P_0 \cup \fixfact{X}{\atoms{P}}) \cdots \square_n P_n : C$.
The coherence of ASP(Q) programs can be defined recursively as follows:
\begin{itemize}
\item
$\exists^{st} P : C$ is coherent, if there exists $M \in SM(P)$ such that $C \cup \fixfact{M}{\atoms{P}}$ is satisfiable; 
\item
$\forall^{st} P : C$ is coherent, if for every $M \in SM(P)$, $C \cup \fixfact{M}{\atoms{P}}$ is satisfiable; 
\item
$\exists^{st} P \Pi$ is coherent, if there exists $M \in SM(P)$ such that $\Pi_{P,M}$ is coherent.
\item
$\forall^{st} P \Pi$ is coherent, if for every $M \in SM(P)$, $\Pi_{P,M}$ is coherent.
\end{itemize}




Before moving to the translations between ASP(Q) and QASP, 
we introduce some special forms of ASP(Q) programs.
We say that an ASP(Q) program of the form (\ref{aspq:program}) is in normal form if
$C$ is empty and the heads of the rules of every $P_i$ do not contain atoms occurring in any $P_j$ such that $j<i$,
and we say that 
an ASP(Q) program is in $\forall$-\emph{GDT} form if 
all its universally quantified logic programs are in \emph{GDT} form.
ASP(Q) programs of the form (\ref{aspq:program}) can be translated to 
ASP(Q) programs in normal and $\forall$-\emph{GDT} form (using auxiliary variables)
following the next steps:
\begin{enumerate}
\item
Replace program $C$ by $\emptyset$, and add $\exists^{st} C$ after $\square_n P_n$.
\item
In every logic program,
replace the normal rules $p \leftarrow B$ 
such that $p$ occurs in some previous program
by constraints $\bot \leftarrow \naf{p}, B$; and
erase the choice rules $\{p\}\leftarrow B$
such that $p$ occurs in some previous program.
\item
Apply the method described in~\cite{niemela08b,famirosc20a}
to translate the universally quantified programs to \emph{GDT} form.
\end{enumerate}
The translation can be computed in polynomial time on the size of the ASP(Q) program,
and the result is an ASP(Q) program in normal and $\forall$-\emph{GDT} form that is coherent
if and only if the original ASP(Q) program is coherent.

%
Once an ASP(Q) program $\Pi$ of the form (\ref{aspq:program}) is in normal and $\forall$-\emph{GDT} form,
we can translate it to a QLP as follows.
Let $qlp(\Pi)$ be the QLP of the form
\[
Q_0 X_0 Q_1 X_1 \ldots Q_n X_n \big( P_0^\bullet \cup P_1^\bullet \cup \ldots \cup P_n^\bullet \cup O \big)
\]
where for every $i \in \{0, \ldots, n\}$,
if $\square_i$ is $\exists^{st}$ then 
$Q_i$ is $\exists$,
$X_i$ is the set of atoms occurring in $P_i$ that do not occur in any previous program, and 
$P_i^{\bullet}$ is $P_0$ if $n=0$ and otherwise it is the program that results from adding the literal $\naf{\flag_i}$ to the body of every rule of $P_i$;
while if $\square_i$ is $\forall^{st}$ then 
$Q_i$ is $\forall$, 
$X_i$ is $\{ p \mid \{p\} \leftarrow \in P_i\}$, and 
$P_i^{\bullet}$ is the program that results from replacing $\bot$ in the head of every
integrity constraint by $\flag_i$; and 
$O$ is $\{ \flag_i \leftarrow \flag_{i-1} \mid \beta \leq i \leq n \}$ where $\beta$ is $1$ if $\square_0=\forall^{st}$ and is $2$ otherwise.
\begin{theorem}
$\Pi$ is coherent if and only if $qlp(\Pi)$ is satisfiable.
\end{theorem}
%


The translation in the other direction is as follows.
Given a QLP $\mathcal{Q}P$ over \sig{A} of the form (\ref{qasp:program}), 
let $aspq(\mathcal{Q}P)$ be the ASP(Q) program of the form
\[
Q^{st}_0 P_0 Q^{st}_1 P_1 \ldots Q^{st}_n P_n \exists^{st}(P \cup O) : \emptyset
\]
where every $P_i$ is $\{ \{ p' \}\leftarrow \mid p \in X_i\}$,
$O$ is 
$\{ \bot \leftarrow p, \naf{p'} \mid p \in \sig{X} \} \cup
\{ \bot \leftarrow \naf{p}, p' \mid p \in \sig{X} \}$ where 
$\sig{X}=\bigcup_{i \in \{0, \dotsc, n\}}X_i$, and
we assume that the set $\{ p' \mid p \in \sig{X} \}$ is disjunct from $\sig{A}$.
\begin{theorem}
$\mathcal{Q}P$ is satisfiable if and only if $aspq(\mathcal{Q}P)$ is coherent.
\end{theorem}

\clearpage
\section{Proof of Theorem \ref{theorem:confp}}

We first introduce some notation.
Let~\dd{} be a planning description.
By \cp{} we denote the logic program
$\facts \cup \ttt{\dd} \cup (\ref{solving:generate}) \cup (\ref{solving:flags})$.
Then (\ref{solving:conformant}) can be rewritten as $\exists Occ \forall \open{}\ \cp{}$.
Let $Generate$ denote $(\ref{solving:generate})$,
and $Alpha$ denote $(\ref{solving:flags})$.
Then $\cp{}$ can be rewritten as 
$\facts \cup \ttt{\dd} \cup Generate \cup Alpha$.
%
%
Since all occurrences of variable $T$ in~$\cp{}$ are always bound by an atom $t(T)$,
and the only ground atoms of that form that can be true are $t(1), \ldots, t(n)$,
we can consider only the ground instantiations of~$\cp{}$ where $T$ has the values from $1$ to $n$.
We add some more notation to represent those ground instances.
We denote by~$r(i)$ the rule obtained by replacing in~$r$ the variable~$T$ by~$i$.
If~$i \leq j$, by $r(i,j)$ we denote the set of rules containing~$r(k)$ for all~$i \leq k \leq j$.
For a set of rules~$R$, we define~$R(i) = \{ r(i) \mid r \in R \}$ and~$R(i,j) = \{ r(i,j) \mid r \in R \}$.
Let $\cpp{}[i]$ denote the program
\begin{gather*}
Alpha(i)
\cup
Generate(i)
\cup
\ttt{\dd_\dynx}(i)
\end{gather*}
and~$\cpp{}[0,j]$ denote the program
\begin{gather*}
Domain
\cup
    \ttt{\dd_\initx}
\cup
\bigcup_{i=1}^j \cp{}[i].
\end{gather*}
Note that~$\cp{}$ 
has the same stable models as the program 
\[\cpp{}[0,n] \cup \ttt{\dd_\finalx}\]
that is the same as 
\[\cpp{}[0,j] \cup \bigcup_{i=j+1}^n \cpp{}[i] \cup \ttt{\dd_\finalx}\]
for every~$j \in \{1,\dots,n\}$.
The following results make repetitive use of the Splitting Theorem~\cite{liftur94a}.
Note that no atom occurring in~$\cpp{}[0,j]$ occurs in the head of a rule in~$\cpp{}[i] \cup \ttt{\dd_\finalx}$ for any~$i > j$.
We also define~$Holds[i] = \{ h(f,i) \mid f \in F \}$.

\begin{lemma}\label{lem:aux0a:theorem:confp}
Let \dd{} be a planning description that represents a conformant planning problem $\pp{}$,
~$s,s' \subseteq F$ be two states,
$a$ be an action,
and $i$ be a positive integer.
Moreover,
let $X$ be a stable model of~$\cpp{}[0,i-1]$
such that~$s' = \{ f \in F \mid h(f,i-1) \in X \}$ and~$\alpha(i-1) \notin X$.
Then, $\trf(s',a) = s$ iff there is a unique stable model~$Y$ of
\begin{gather}
    X \cup \cpp{}[i]  \cup fixcons(\{occ(a,i)\},Occ_i)
    \label{e1:0:lem:aux0a:theorem:confp}
\end{gather}
such that~$s  = \{ f \in F \mid h(f,i) \in Y \}$ and $\alpha(i) \notin Y$.
Furthermore,
if $\trf(s',a) = \bot$ then (\ref{e1:0:lem:aux0a:theorem:confp}) is unsatisfiable.
\end{lemma}

\begin{proof}
    Let
    \begin{align*}
        P_1 &= X \cup Alpha(i) \cup Generate(i) \cup fixcons(\{occ(a,i)\},Occ_i) 
        \\
        P_2 &= P_1 \cup \ttt{\dd_\dynx}(i)
    \end{align*}
    Recall that~$\cpp{}[i] =   Alpha(i) \cup Generate(i) \cup \ttt{\dd_\dynx}(i)$ and, thus,~$P_2$ is equal to~\eqref{e1:0:lem:aux0a:theorem:confp}.
    Program~$P_1$ has a unique stable model~$X_1 = X \cup \{occ(a,i)\}$
    and, by the Splitting Theorem, the stable models of~$P_2$ are the same as the stable models of~$X_1 \cup \ttt{\dd_\dynx}(i)$.
    Note that~$\alpha(i-1) \notin X$ implies that~$\alpha(i) \notin X_1$ because the only rule with~$\alpha(i)$ in the head has~$\alpha(i-1)$ in the body.
    Furthermore, since~\dd{} represents~$\pp{}$ and~$\alpha(i) \notin X$, it follows that 
    \begin{gather}
        (X \cap Holds[i-1]) \cup \{occ(a,i)\} \cup \{ t(i) \} \cup \ttt{\dd_\dynx}(i)
        \label{eq:1:lem:aux0a:theorem:confp} 
    \end{gather}
    either (Case 1) has a unique stable model~$Y'$ such that~${s = \{ f \in F \mid h(f,i) \in Y' \}}$ or (Case 2) has no stable model and~${\trf(s',a) = \bot}$.
    Recall that~$\dyn{\dd{}}$ is deterministic.
    Hence, it cannot generate multiple stable models, and the same applies to $\ttt{\dd_\dynx}(i)$.
    Note also that~$\alpha(i) \notin X$ implies~$\alpha(i) \notin Y'$ because the only rules with~$\alpha(i)$ in the head belong to~$P_1$.
    In Case 1 the iff holds because $Y = Y' \cup X$ is the unique stable model of~$X_1 \cup \ttt{\dd_\dynx}(i)$, and 
    it satisfies the conditions ${s = \{ f \in F \mid h(f,i) \in Y \}}$ and~$\alpha(i) \notin Y$.
    To see this, observe that the atoms in~$X_1$ that do not belong to~\eqref{eq:1:lem:aux0a:theorem:confp}
    also do not occur in the rules of~$\ttt{\dd_\dynx}(i)$.
    In Case 2 the iff holds because none of its two sides holds.
    Finally, if ${\trf(s',a) = \bot}$ then (\ref{eq:1:lem:aux0a:theorem:confp}) is unsatisfiable, hence~$X_1 \cup \ttt{\dd_\dynx}(i)$
    is also unsatisfiable, and therefore (\ref{e1:0:lem:aux0a:theorem:confp}) is unsatisfiable.%
\end{proof}

\begin{lemma}\label{lem:aux0b:theorem:confp}
Let \dd{} be a planning description that represents a conformant planning problem $\pp{}$,
$s,s' \subseteq F$ be two states, 
$p= [ a_1; a_2; \dotsc ;a_{n} ]$ be a plan,
$i$ be a positive integer, and
$X$ be a stable model of~$\cpp{}[0,i-1]$
such that~$s' = \{ f \in F \mid h(f,i-1) \in X \}$ and~$\alpha(i-1) \notin X$.
Moreover, let
\begin{align*}
    A &= \{ occ(a_1,i), occ(a_2,i+1), \dotsc, occ(a_n,i+n-1) \}
\end{align*}
be a set of atoms representing action occurrences.
Then, $\trf(\{s'\},p) = s$ iff there is a unique stable model~$Y$ of
\begin{gather}
    X \cup \bigcup_{j=i}^{i+n-1} \big( \cpp{}[j]  \cup fixcons(\{occ(a,j)\},Occ_j) \big)
    \label{eq:0:lem:aux0b:theorem:confp}
\end{gather}
such that~$s  = \{ f \in F \mid h(f,i + n -1) \in Y \}$. Furthermore,~$\alpha(i + n -1) \notin Y$.
\end{lemma}

\begin{proof}
We consider first the case where $\trf(s',a_1)$ is $\bot$, and then the case where $\trf(s',a_1)$ is a state $s''$.
If $\trf(s',a_1)$ is $\bot$, then by Lemma~\ref{lem:aux0a:theorem:confp}
the program
\begin{gather}
X \cup \cpp{}[i] \cup fixcons(\{occ(a,i)\},Occ_i)
    \label{eq:1:lem:aux0b:theorem:confp}
\end{gather}
is unsatisfiable.
Hence, by the Splitting Theorem, the program 
(\ref{eq:0:lem:aux0b:theorem:confp}) is also unsatisfiable, and the iff holds because none of its sides holds.
We continue with the case where $\trf(s',a_1)$ is a state $s''$.
If~$n=1$, the lemma statement follows directly from Lemma~\ref{lem:aux0a:theorem:confp}.
Then the proof follows by induction assuming that the lemma statement holds for all plans of length~$n-1$.
By definition, $\trf(\{s'\},p) = \trf(\{s''\},q)$ with~$q=[ a_2; \dotsc; a_{n}]$.
From Lemma~\ref{lem:aux0a:theorem:confp},
we get that~$\trf(s',a_1) = s''$ iff there is a unique stable model~$Z$ of
(\ref{eq:1:lem:aux0b:theorem:confp})
such that~$s''  = \{ f \in F \mid h(f,i) \in Z \}$ and~$\alpha(i) \notin Z$.
%
%
%
By induction hypothesis, we get that~$\trf(\{s''\},q) = s$ iff there is a unique stable model~$Y$ of
$$Z \cup \bigcup_{j=i+1}^{i+n-1} \big( \cpp{}[j]  \cup fixcons(\{occ(a,j)\},Occ_j) \big)$$
such that~$s  = \{ f \in F \mid h(f,i+n-1) \in Y \}$.
Furthermore,~$\alpha(i + n -1) \notin Y$.
From the Splitting Theorem,~$Y$ is the unique stable model of~\eqref{eq:0:lem:aux0b:theorem:confp}, 
and the iff condition holds.
\end{proof}

\begin{lemma}\label{lem:aux0:theorem:confp}
Let \dd{} be a planning description that represents a conformant planning problem $\pp{}$.
Let~$Y \subseteq Open$ be a set of atoms,~$s \subseteq F$ and~$s_0 \in SM(\init{\dd{}})$ be states such that
$$Y = \{ h(f,0) \in Open \mid f \in s_0 \},$$ and $p= [ a_1; a_2; \dotsc; a_{n}]$ be a plan
with
\begin{align*}
    X &= \{ occ(a_1,1), occ(a_2,2), \dotsc, occ(a_n,n) \}.
\end{align*}
Then, $\trf(\{s_0\},p) = s$ iff there is a unique stable model~$Z$ of
\begin{gather}
    \cpp{}[0,n] \cup \fixcons{X}{Occ} \cup \fixcons{ Y }{Open}
    \label{eq:0:lem:aux0:theorem:confp}
\end{gather}
such that~$s  = \{ f \mid h(f,n) \in Z \}$. Furthermore, $\alpha(n) \notin Z$.
\end{lemma}

\begin{proof}
Note that~$\cpp{}[0,0] = Domain \cup \ttt{\dd_\initx}$ and that
$s_0 \in SM(\init{\dd{}})$ implies that this program has a unique stable model~$Z_0$ such that~$s_0 = \{ f \in F \mid h(f,0) \in Z_0 \}$ and~$\alpha(0) \notin Z_0$.
Then, from Lemma~\ref{lem:aux0b:theorem:confp} with~$i = 1$, we get that
$\trf(\{s_0\},p) = s$ iff there is a unique stable model~$Z$ of
$$Z_0 \cup \bigcup_{j=1}^{n} \big( \cpp{}[j]  \cup fixcons(\{occ(a,j)\},Occ_j) \big).$$
Furthermore, $\alpha(n) \notin Z$.
From the Splitting Theorem, this implies that~$Z$ is a stable model of~$\cpp{}[0,n]$.
Furthermore, it is the unique stable model of~$\cpp{}[0,n]$ such that
$$
s_0 = \{ f \in F \mid h(f,0) \in Z_0 \} = \{ f \in F \mid h(f,0) \in Z \}.
$$
Note also that~$Y = \{ h(f,0) \in Open \mid f \in s_0 \} = Z \cap Open$
and, thus, $Z$ satisfies
$$\fixcons{ Y}{Open}
 = \fixcons{ Z \cap Open }{Open}.$$
No other stable model of~$\cpp{}[0,0]$ satisfies these constraints.
Therefore, $Z$ is the unique stable model of~\eqref{eq:0:lem:aux0:theorem:confp}.
\end{proof}

\begin{lemma}\label{lem:aux1a:theorem:confp}
Let \dd{} be a planning description that represents a conformant planning problem $\pp{}$ such that~$\init{\dd}$ is in $GDT$ form, 
and $n$ be a positive integer.
Let $p= [ a_1; \dotsc; a_n ]$ be a plan, 
$X = \{ occ(a_1,1), \dotsc occ(a_n,n) \}$ be a set of atoms, and
$Y \subseteq Open$ be another set of atoms.
Then, the following two statements are equivalent:
\begin{itemize}
    \item there is some stable model~$Z$ of $\cp \cup \fixcons{X}{Occ} \cup \fixcons{Y}{Open}$ such that~$\alpha(n) \notin Z$; and
    \item there is $s_0 \in SM(\init{\dd})$ s.t.~$Y = \{ h(f,0) \in Open \mid f \in s_0 \}$ and~$\trf_{\dyn{\dd}}(\{s_0\},p) \subseteq G$.%
\end{itemize}
\end{lemma}

\begin{proof}
First, assume that there is no~$s_0 \in SM(\init{\dd})$ s.t.~$Y = \{ h(f,0) \in Open \mid f \in s_0 \}$.
In this case, the second lemma statement does not hold, and we will see next that the first lemma statement
does not hold either.
Given our assumption and by construction,
the program $\base \cup \ttt{\dd_\initx} \cup Alpha \cup \fixcons{Y}{Open}$, 
that is a subset of $\cp \cup \fixcons{X}{Occ} \cup \fixcons{Y}{Open}$,
has a unique stable model, and this stable model contains the atom $\alpha(n)$.
Then, by the Splitting Theorem, the stable models of $\cp \cup \fixcons{X}{Occ} \cup \fixcons{Y}{Open}$
contain the atom $\alpha(n)$, which contradicts the first lemma statement.

Otherwise, there is $s_0 \in SM(\init{\dd})$ such that~$Y = \{ h(f,0) \in Open \mid f \in s_0 \}$.
Then, from Lemma~\ref{lem:aux0:theorem:confp}, it follows that
$\trf_{\dyn{\dd}}(\{s_0\},p) = s$ iff there is a unique stable model~$Z$ of
$$\cpp{}[0,n] \cup \fixcons{X}{Occ} \cup \fixcons{Y}{Open}$$
such that~$s  = \{ f \in F \mid h(f,n) \in Z \}$. Furthermore, $\alpha(n) \notin Z$.
Note that~$\cp = \cpp{}[0,n]$.
Moreover, we can see that~$s \in G$ iff~$Z$ satisfies~$\ttt{\dd_\finalx}$.
Since~${\ttt{\dd_\finalx} \subseteq \ttt{\dd}}$, it follows that the two lemma statements are equivalent.%
\end{proof}

\begin{lemma}\label{lem:aux1b:theorem:confp}
Let \dd{} be a planning description that represents a conformant planning problem $\pp{}$ such that $\init{\dd}$ is in $GDT$ form.
Then, the following statements hold:
\begin{itemize}
    \item For every $Y\subseteq Open$ such that~$\ttx{\dd_\initx} \cup  \fixcons{Y}{Open}$ is satisfiable,
    there is a unique~$s_Y \in SM(\init{\dd})$ such that~$Y = \{ h(f,0) \in Open \mid f \in s_Y \}$.
    \item For every~$s_0 \in SM(\init{\dd})$, there is a unique~$Y\subseteq Open$ such that \\ $Y = \{ h(f,0) \in Open \mid f \in s_0 \}$.
    Furthermore, $\ttx{\dd_\initx} \cup \fixcons{Y}{Open}$ is satisfiable.
\end{itemize}
\end{lemma}

\begin{proof}
    For the first statement, note that there is a choice rule of form~$\{ f \} \leftarrow$ in~$\init{\dd}$ iff the atom~$h(f,0)$ belongs to~$Open$.
    The rest of the program is deterministic, therefore, there is at most one $s_0 \in SM(\init{\dd})$ such that~$Y = \{ h(f,0) \in Open \mid f \in s_0 \}$.
    Furthermore, since~$\ttx{\dd_\initx} \cup  \fixcons{Y}{Open}$ is satisfiable, such stable model exists.
    For the same reason, for the second statement, if~$s_0 \in SM(\init{\dd})$, there is a unique~$Y\subseteq Open$ such that~$Y = \{ h(f,0) \in Open \mid f \in s_0 \}$.
    Furthermore, it is easy to see that~$\ttx{\dd_\initx} \cup \fixcons{Y}{Open}$ is satisfiable.
\end{proof}

\begin{lemma}\label{lem:aux1:theorem:confp}
Let \dd{} be a planning description that represents a conformant planning problem $\pp{}$ such that $\init{\dd}$ is in $GDT$ form, 
and $n$ be a positive integer.
Let $p= [ a_1; \dotsc; a_n ] $ be a plan,
and $X = \{ occ(a_1,1), \dotsc occ(a_n,n) \}$ be a set of atoms.
Then, the following two statements are equivalent:
\begin{itemize}
    \item for all $Y\subseteq Open$ such that~$\ttx{\dd_\initx} \cup  \fixcons{Y}{Open}$ is satisfiable,
    there is some stable model~$Z$ of~$\cp \cup \fixcons{X}{Occ} \cup \fixcons{Y}{Open})$ such that~$\alpha(n) \notin Z$; and
    \item every $s_0 \in SM(\init{\dd})$ 
    satisfies $\trf_{\dyn{\dd}}(\{s_0\},p) \subseteq G$.
\end{itemize}
\end{lemma}

\begin{proof}
From Lemma~\ref{lem:aux1a:theorem:confp},
the first lemma statement holds
iff
\begin{flalign*}
    \hspace{0.5cm}&\parbox{325pt}{for all $Y\subseteq Open$ such that~$\ttx{\dd_\initx} \cup  \fixcons{Y}{Open}$ is satisfiable,\\
    there is $s \in SM(\init{\dd})$ such that~$Y = \{ h(f,0) \in Open \mid f \in s \}$\\ and~$\trf_{\dyn{\dd}}(\{s\},p) \subseteq G$.}&
\end{flalign*}
Furthermore, for every $Y\subseteq Open$,
there is at most one state~$s_0 \in SM(\init{\dd})$ such that~$Y = \{ h(f,0) \in Open \mid f \in s_0 \}$ (Lemma~\ref{lem:aux1b:theorem:confp}).
Therefore, the above statement can be rewritten as
\begin{flalign}
    \hspace{0.5cm}&\parbox{325pt}{every ${Y\subseteq Open}$ satisfies that, if~${\ttx{\dd_\initx} \cup  \fixcons{Y}{Open}}$ is satisfiable,\\ then~$\trf_{\dyn{\dd}}(\{s_Y\},p) \subseteq G$ where~$s_Y \in SM(\init{\dd})$ is the unique state such that~${Y = \{ h(f,0) \in Open \mid f \in s_Y \}}$.}&
    \label{eq:1:lem:aux1:theorem:confp}
\end{flalign}
\emph{First-to-second statement}.
Pick any~$s_0 \in SM(\init{\dd})$.
From Lemma~\ref{lem:aux1b:theorem:confp},
there is a unique set~$Y\subseteq Open$ s.t.~${Y = \{ h(f,0) \in Open \mid f \in s_0 \}}$.
Furthermore, by construction, $\ttx{\dd_\initx} \cup  \fixcons{Y}{Open}$ is satisfiable.
From~\eqref{eq:1:lem:aux1:theorem:confp}, this implies~$\trf_{\dyn{\dd}}(\{s_0\},p) \subseteq G$ and, thus, the second lemma statement holds.
\\[5pt]
\emph{Second-to-first statement}.
Pick any $Y\subseteq Open$ such that~$\ttx{\dd_\initx} \cup  \fixcons{Y}{Open}$ is satisfiable.
From Lemma~\ref{lem:aux1b:theorem:confp}, this implies that there is a unique~$s_Y \in SM(\init{\dd})$ such that~${Y = \{ h(f,0) \in Open \mid f \in s_Y \}}$.
From the second lemma statement, this implies that~$\trf_{\dyn{\dd}}(\{s_Y\},p) \subseteq G$.
Hence, \eqref{eq:1:lem:aux1:theorem:confp} and the first lemma statement hold.%
\end{proof}

\begin{lemma}\label{lem:aux3:theorem:confp}
Let \dd{} be a planning description that represents a conformant planning problem $\pp{}$ 
such that $\init{\dd}$ is in $GDT$ form, and $n$ be a positive integer.
Let $p = [ a_0; \dotsc; a_n ]$ be a plan, 
$X = \{ occ(a_1,1), \dotsc occ(a_n,n) \}$ be a set of atoms, and
$Y \subseteq Open$.
Then, the following two statements are equivalent:
\begin{itemize}
    \item $\cp \cup \fixcons{X}{Occ} \cup \fixcons{Y}{Open}$ is satisfiable; and
    \item if~$\ttx{\dd_\initx} \cup  \fixcons{Y}{Open}$ is satisfiable, 
     then there is some stable model~$Z$ of~$\cp \cup \fixcons{X}{Occ} \cup \fixcons{Y}{Open})$ such that~$\alpha(n) \notin Z$.
\end{itemize}
\end{lemma}
\begin{proof}
We proceed by cases.
Assume first that~$\ttx{\dd_\initx} \cup  \fixcons{Y}{Open}$ is satisfiable.
Then, the second lemma statement trivially implies the first.
Next, assume that the first lemma statement holds.
We show that the second lemma statement follows.
By definition, the first lemma statement implies that there is some stable model~$Z$ of~$\cp \cup \fixcons{X}{Occ} \cup \fixcons{Y}{Open})$.
Hence, we only need to show that~$\alpha(n) \notin Z$ and, for this, it is enough to show that~$\alpha(0) \notin Z$.
Suppose, for the sake of contradiction, that~$\alpha(0) \in Z$.
From the Splitting Theorem, there is a stable model~$Z_0$ of~$\ttt{\dd_\initx} \cup  \fixcons{Y}{Open}$ such that~${\alpha(0) \in Z_0}$.
By construction, this implies that~$\ttx{\dd_\initx} \cup  \fixcons{Y}{Open}$ is not satisfiable, which is a contradiction with the assumption.

Assume now that~$\ttx{\dd_\initx} \cup  \fixcons{Y}{Open}$ is not satisfiable.
Then, the second lemma statement trivially holds.
Furthermore, by construction, every stable model~$Z_0$ of~$\ttt{\dd_\initx} \cup  \fixcons{Y}{Open}$ satisfies~${\alpha(0) \in Z_0}$.
Let~$Z = Z_0 \cup X \cup \{ \alpha(i) \mid 1 \leq i \leq n \}$.
It holds that
\begin{itemize}
    \item the bodies of all rules in~$\ttt{\dd_\dynx} \cup \ttt{\dd_\finalx}$ are not satisfied by~$Z$,
    \item the bodies of all rules in~$Alpha$ are satisfied by~$Z$ and all~$\alpha(i)$ are thus supported, and
    \item all rules in~$Generate$ are satisfied by~$Z$ and all elements in~$X$ are supported.
\end{itemize}
Thus~$Z$ is a stable model of~$\cp \cup \fixcons{X}{Occ} \cup \fixcons{Y}{Open}$ and the first lemma statement holds.
\end{proof}

\begin{lemma}\label{lem:aux4:theorem:confp}
Let \dd{} be a planning description that represents a conformant planning problem $\pp{}$ such that $\init{\dd}$ is in $GDT$ form, and $n$ be a positive integer.
Let $p = [ a_0; \dotsc; a_n ]$ be a plan and~$X = \{ occ(a_0,0), \dotsc occ(a_n,n) \}$ be a set of atoms.
Then, the following two statements are equivalent:
\begin{itemize}
    \item for all $Y\subseteq Open$,
    $\cp \cup \fixcons{X}{Occ} \cup \fixcons{Y}{Open}$ is satisfiable; and
    \item every $s_0 \in SM(\init{\dd})$ 
    satisfies $\trf_{\dyn{\dd}}(\{s_0\},p) \subseteq G$.
\end{itemize}
\end{lemma}

\begin{proof}
The first lemma statement holds
\\iff
for all $Y\subseteq Open$ satisfying that~$\ttx{\dd_\initx} \cup  \fixcons{Y}{Open}$ is satisfiable, 
there is some stable model~$Z$ of~$\cp \cup \fixcons{X}{Occ} \cup \fixcons{Y}{Open})$ such that~$\alpha(n) \notin Z$
(Lemma~\ref{lem:aux3:theorem:confp})
\\iff
the second lemma statement holds
(Lemma~\ref{lem:aux1:theorem:confp}).%
\end{proof}

\begin{lemma}\label{lem:aux5:theorem:confp}
Let \dd{} be a planning description that represents a conformant planning problem $\pp{}$ 
such that $\init{\dd}$ is in $GDT$ form, and let~$n$ be a positive integer.
Then, the following two statements are equivalent:
\begin{itemize}
    \item 
    there is some $X\subseteq Occ$ such that, for all $Y\subseteq Open$,
    there is some stable model of~$\cp \cup \fixcons{X}{Occ} \cup \fixcons{Y}{Open}$; and
    \item 
    there is some $X\subseteq Occ$ of the form $\{occ(a_1,1), \ldots, occ(a_n,n)\}$ such that, for all $Y\subseteq Open$,
    there is some stable model of $\cp \cup \fixcons{X}{Occ} \cup \fixcons{Y}{Open}$.
\end{itemize}
\end{lemma}
\begin{proof}
The second statement trivially implies the first.
In the other direction, assume that the first statement holds, we show that in this case the second statement also holds.
Pick any $Y\subseteq Open$, and let $M$ be some stable model of $\cp \cup \fixcons{X}{Occ} \cup \fixcons{Y}{Open}$.
Given that the rule $Generate$ belongs to $\cp$,
it follows that $M \cap Occ$ has the form $\{occ(a_1,1), \ldots, occ(a_n,n)\}$.
On the other hand, since $M$ satisfies $\fixcons{X}{Occ}$, it follows that $M \cap Occ = X$.
Hence, $X$ has the form $\{occ(a_1,1), \ldots, occ(a_n,n)\}$, and the second statement holds.
\end{proof}

\begin{proof*}[Proof of Theorem \ref{theorem:confp}]
Since~\dd{} represents~$\pp{}$,
it follows that
$\pp{} = \langle \trf_{\dyn{\dd{}}}, SM(\init{\dd{}}), G \rangle$
where~$G$ is $SM( \final{\dd{}} \cup \{ \{f\}\leftarrow \mid f \in F \})$.
Then, the program
\[\exists Occ \forall \open{}\ \cp{}\]
is satisfiable
\\iff
there is some $X\subseteq Occ$ such that, 
for all $Y\subseteq Open$,
there is some stable model of $\cp \cup \fixcons{X}{Occ} \cup \fixcons{Y}{Open}$
\\iff
there is some $X\subseteq Occ$ of the form $\{occ(a_1,1), \ldots, occ(a_n,n)\}$ such that,
for all $Y\subseteq Open$,
there is some stable model of $\cp \cup \fixcons{X}{Occ} \cup \fixcons{Y}{Open}$
\hfill (Lemma~\ref{lem:aux5:theorem:confp})
\\iff
there is some $X\subseteq Occ$ of the form $\{occ(a_1,1), \ldots, occ(a_n,n)\}$ such that every $s_0 \in SM(\init{\dd})$ 
satisfies $\trf_{\dyn{\dd}}(\{s_0\},p) \subseteq G$ with~$p= [ a_0; \dotsc; a_n ]$
\hfill (Lemma~\ref{lem:aux4:theorem:confp})
\\iff
there is some plan $p$ of length $n$ s.t. every $s_0 \in SM(\init{\dd{}})$ 
satisfies~${\trf_{\dyn{\dd{}}}(\{s_0\},p) \subseteq G}$ 
\\iff
there is some plan $p$ of length $n$ such that $\trf_{\dyn{\dd{}}}(SM(\init{\dd{}}),p) \subseteq G$
\\iff
there is some plan of length $n$ that solves problem~$\pp{}$.$\mathproofbox$
\end{proof*}

\end{document}
